\documentclass{article}

\usepackage{microtype}
\usepackage{graphicx}
\usepackage{subfigure}
\usepackage{booktabs} % for professional tables

\usepackage{amsmath}
\usepackage{mathrsfs}
\usepackage{amsfonts}
\usepackage{amsthm}
\usepackage{stmaryrd}
\usepackage{xcolor}
\usepackage{colortbl}
\usepackage{cases}
\usepackage{multirow}
\usepackage{chngcntr}
\usepackage{enumitem}
\usepackage{mathtools}
\usepackage{booktabs}
\usepackage{thmtools,thm-restate}
\usepackage{floatrow}
\newfloatcommand{capbtabbox}{table}[][\FBwidth]
\usepackage{algorithm}
\usepackage{algorithmic}

\usepackage{trackchanges}
\addeditor{Will}
\addeditor{Georg}
\addeditor{Remi}
\addeditor{Ivo}

\newcommand{\cX}{\mathcal{X}}

\newcommand{\cZ}{\mathcal{Z}}

\newcommand{\cL}{\mathcal{L}}

\newcommand{\bR}{\mathbb{R}}

\newcommand{\bN}{\mathbb{N}}

\DeclareMathOperator*{\expect}{{\huge \mathbb{E}}}

\DeclareMathOperator*{\argmin}{arg\,min}

\renewcommand{\phi}{\varphi}
\renewcommand{\epsilon}{\varepsilon}

\newcommand{\U}{\mathcal{U}}

\newcommand{\1}{\mathbb{I}}

\newcommand{\beq}{\begin{equation}}
\newcommand{\eeq}{\end{equation}}

\newcommand{\beqa}{\begin{eqnarray}}
\newcommand{\eeqa}{\end{eqnarray}}

\newcommand{\beqan}{\begin{eqnarray*}}
\newcommand{\eeqan}{\end{eqnarray*}}

\usepackage{verbatim}
\usepackage{cancel}
\usepackage{thmtools,thm-restate}

\newtheorem{proposition}{Proposition}

\usepackage[accepted]{icml2018}

\icmltitlerunning{Autoregressive Quantile Networks for Generative Modeling}

\begin{document}

\twocolumn[
\icmltitle{Autoregressive Quantile Networks for Generative Modeling}



\icmlsetsymbol{equal}{*}

\begin{icmlauthorlist}
\icmlauthor{Georg Ostrovski}{equal,dm}
\icmlauthor{Will Dabney}{equal,dm}
\icmlauthor{R\'emi Munos}{dm}
\end{icmlauthorlist}

\icmlaffiliation{dm}{DeepMind, London, UK}

\icmlcorrespondingauthor{Georg Ostrovski}{ostrovski@google.com}
\icmlcorrespondingauthor{Will Dabney}{wdabney@google.com}

\icmlkeywords{autoregressive, generative model, machine learning, quantile regression}

\vskip 0.3in
]



\printAffiliationsAndNotice{\icmlEqualContribution} % otherwise use the standard text.

\begin{abstract}
We introduce autoregressive implicit quantile networks (AIQN), a fundamentally different approach to generative modeling than those commonly used, that implicitly captures the distribution using quantile regression. AIQN is able to achieve superior perceptual quality and improvements in evaluation metrics, without incurring a loss of sample diversity. The method can be applied to many existing models and architectures. In this work we extend the PixelCNN model with AIQN and demonstrate results on CIFAR-10 and ImageNet using Inception score, FID, non-cherry-picked samples, and inpainting results. We consistently observe that AIQN yields a highly stable algorithm that improves perceptual quality while maintaining a highly diverse distribution.
\end{abstract}

\section{Introduction}
\label{sec:introduction}

There has been a staggering increase in progress on generative modeling in recent years, built largely upon fundamental advances such as generative adversarial networks \cite{goodfellow2014generative}, variational inference \cite{kingma2013auto}, and autoregressive density estimation \cite{vandenoord16pixel}. These have led to breakthroughs in state-of-the-art generation of natural images \cite{karras2017progressive} and audio \cite{van2016wavenet}, and even been used for unsupervised learning of disentangled representations \cite{higgins2016beta,chen2016infogan}. These domains often have real-valued distributions with underlying metrics; that is, there is a domain-specific notion of similarity between data points. This similarity is ignored by the predominant work-horse of generative modeling, the Kullback-Leibler (KL) divergence. Progress is now being made towards algorithms that optimize with respect to these underlying metrics \cite{wgan,bousquet2017optimal}.

%(1) novel approach to generative modeling
In this paper, we present a novel approach to generative modeling, that, while strikingly different from existing methods, is grounded in the well-understood statistical methods of \textit{quantile regression}. Unlike the majority of recent work, we approach generative modeling without the use of the KL divergence, and without explicitly approximating a likelihood model. Like GANs, in this way we produce an implicitly defined model, but unlike GANs our optimization procedure is inherently stable and lacks degenerate solutions which cause loss of diversity and mode collapse.

%(2) receives diversity benefits of pixelcnn and perceptually-sensitive loss of GANs (not quite how I want to say it)
Much of the recent research on GANs has been focused on improving stability \cite{radford2015unsupervised,wgan,daskalakis2017training} and sample diversity \cite{gulrajani2017improved,salimans2016improved,salimans2018otgan}. By stark contrast, methods such as PixelCNN \cite{van2016conditional} readily produce high diversity, but due to their use of KL divergence are unable to make reasonable trade-offs between likelihood and perceptual similarity \cite{theis2015note,bellemare17cramer,bousquet2017optimal}.

%(3) simple
%(4) robust
Our proposed method, \textit{autoregressive implicit quantile networks} (AIQN), combines the benefits of both: a loss function that respects the underlying metric of the data leading to improved perceptual quality, and a stable optimization process leading to highly diverse samples. While there has been an increasing tendency towards complex architectures \cite{chen2017pixelsnail,salimans2017pcnn} and multiple objective loss functions to overcome these challenges, AIQN is conceptually simple and does not rely on any special architecture or optimization techniques. Empirically it proves to be robust to hyperparameter variations and easy to optimize.

%(5) works well, and can be applied more widely than competing methods.
Our work is motivated by the recent advances achieved by reframing GANs in terms of optimal transport, leading to the Wasserstein GAN algorithm \cite{wgan}, as well as work towards understanding the relationship between optimal transport and both GANs and VAEs \cite{bousquet2017optimal}. In agreement with these results, we focus on loss functions grounded in perceptually meaningful metrics. We build upon recent work in distributional reinforcement learning \cite{iqn2018}, which has begun to bridge the gap between approaches in reinforcement learning and unsupervised learning. Towards a practical algorithm we base our experimental results on Gated PixelCNN \cite{van2016conditional}, and show that using AIQN significantly improves objective performance on CIFAR-10 and ImageNet 32x32 in terms of Fr\'echet Inception Distance (FID) and Inception score, as well as subjective perceptual quality in image samples and inpainting.

\section{Background}
\label{sec:background}
We begin by establishing some notation, before turning to a review of three of the most prevalent methods for generative modeling. 
%As in \citet{bousquet2017optimal}, 
Calligraphic letters (e.g.~$\cX$) denote sets or spaces, capital letters (e.g. $X$) denote random variables, and lower case letters (e.g.~$x$) indicate values. A probability distribution with random variable $X \in \cX$ is denoted $p_X \in\mathscr{P}(\cX)$, its cumulative distribution function (c.d.f.) $F_X$, and inverse c.d.f.~or quantile function $Q_X = F^{-1}_X$. When probability distributions or quantile functions are parameterized by some $\theta$ we will write $p_\theta$ or $Q_\theta$ recognizing that here we do not view $\theta$ as a random variable.

Perhaps the simplest way to approach generative modeling of a random variable $X \in \cX$ is by fixing some discretization of $\cX$ into $n$ separate values, say $x_1, \ldots, x_n \in \cX$, and parameterize the approximate distribution with $p_\theta(x_i) \propto \exp(\theta_i)$. This type of categorical parameterization is widely used, only slightly less commonly when $\cX$ does not lend itself naturally to such a partitioning. Typically, the parameters $\theta$ are optimized to minimize the Kullback-Leibler (KL) divergence between observed values of $X$ and the model $p_\theta$, $\theta^* = \argmin_\theta D_{KL}(p_X \| p_\theta)$.

However, this is only tractable when $\cX$ is a small discrete set or at best low-dimensional. A common method for extending a generative model or density estimator to multivariate distributions is to factor the density as a product of scalar-valued conditional distributions. Let $X = (X_1, \ldots, X_n)$, then for any permutation of the dimensions $\sigma\colon \bN_n \to \bN_n$,
\begin{equation}\label{eqn:conditional_factor}
p_X(x) = \prod_{i=1}^n p_{X_{\sigma(i)}}(x_{\sigma(i)} | x_{\sigma(1)}, \ldots, x_{\sigma(i-1)}).
\end{equation}

When the conditional density is modeled by a simple (e.g.~Gaussian) base distribution, the ordering of the dimensions can be crucial \cite{papamakarios2017masked}. However, it is common practice to choose an arbitrary ordering and rely upon a more powerful conditional model to avoid these problems. This class of models includes PixelRNN and PixelCNN \cite{vandenoord16pixel,van2016conditional}, MAF \cite{papamakarios2017masked}, MADE \cite{germain2015made}, and many others. Fundamentally, all these approaches use the KL divergence as their loss function.

Another class of methods, generally known as \textit{latent variable methods}, can bypass the need for autoregressive models using a different modeling assumption. Specifically, consider the Variational Autoencoder (VAE) \cite{kingma2013auto,rezende2014stochastic}, which represents $p_\theta$ as the marginalization over a latent random variable $Z \in \cZ$. The VAE is trained to maximize an approximate lower bound of the log-likelihood of the observations: %% Need a more detailed and cleaner explanation of VAEs 
\begin{equation*}
	\log p_\theta(x) \ge - D_{KL}(q_\theta(z | x) \| p(z)) + \expect \left[ \log p_\theta(x | z) \right].
\end{equation*}

Although VAEs are straightforward to implement and optimize, and effective at capturing structure in high-dimensional spaces, they often miss fine-grained detail, resulting in blurry images. 

Generative Adversarial Networks (GANs) \cite{goodfellow2014generative} pose the problem of learning a generative model as a two-player zero-sum game between a discriminator $D$, attempting to distinguish between $x \sim p_X$ (real data) and $x \sim p_\theta$ (generated data), and a generator $G$, attempting to generate data indistinguishable from real data. The generator is an implicit latent variable model that reparameterizes samples, typically from an isotropic Gaussian distribution, into values in $\cX$. The original formulation of GANs, 
\begin{equation*}
    \argmin_G \sup_D \left[ \expect_X \log(D(X)) + \expect_Z \log(1 - D(G(Z))) \right],
\end{equation*}
can be seen as minimizing a lower-bound on the Jensen-Shannon divergence \cite{goodfellow2014generative,bousquet2017optimal}. That is, even in the case of GANs we are often minimizing functions of the KL divergence\footnote{The Jensen-Shannon divergence is the sum of KLs between distributions $P, Q$ and their uniform mixture $M = 0.5 (P+Q)$: $\operatorname{JSD}(P||Q) = 0.5 (D_{KL}(P || M) + D_{KL}(Q || M))$.}.

Many recent advances have come from principled combinations of these three fundamental methods \cite{makhzani2015adversarial,dumoulin2016adversarially,rosca2017variational}.

\begin{figure}[t]
\begin{center}
\includegraphics[width=\textwidth]{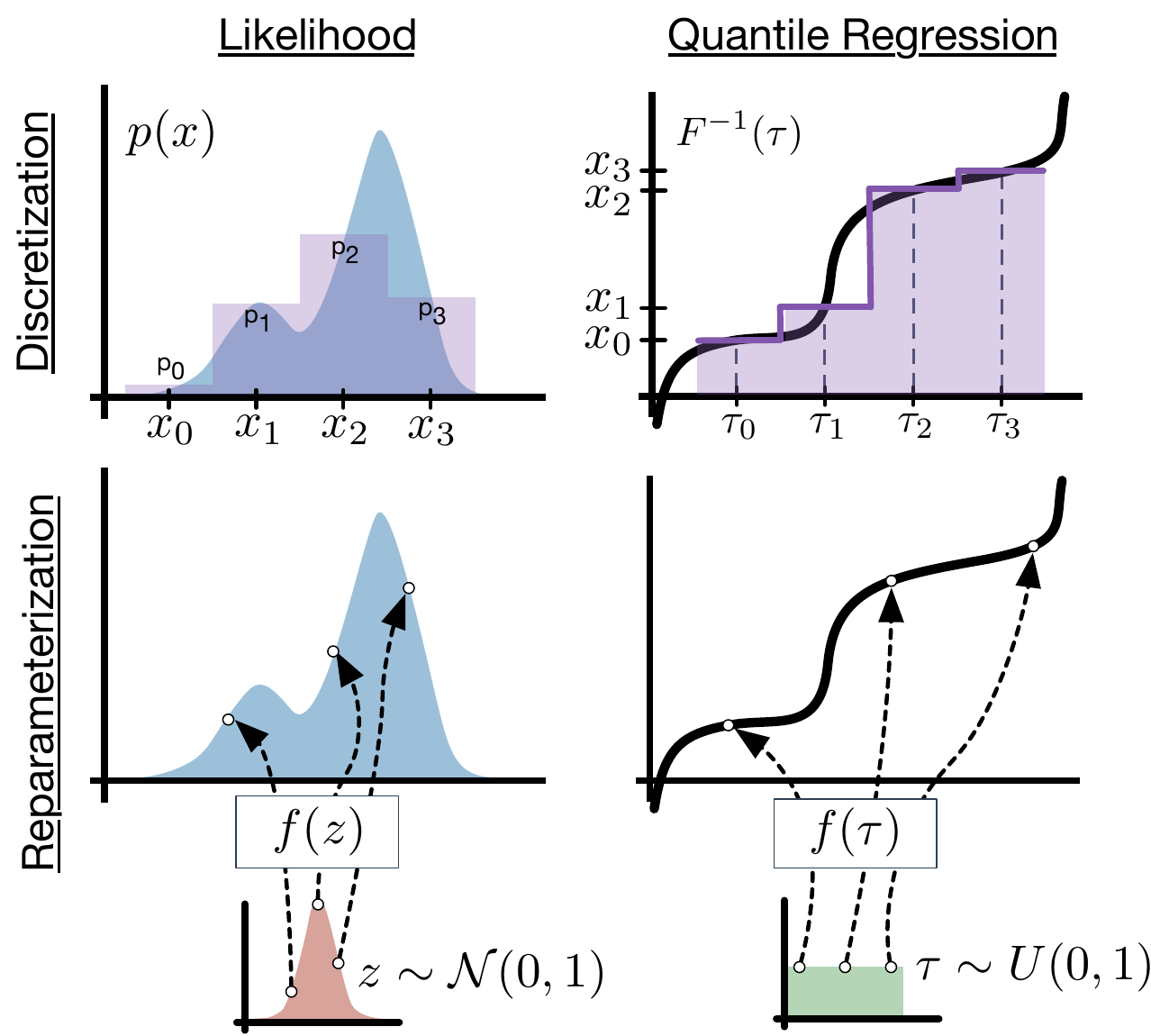}
\end{center}
\caption{Categorization of generative models by discretization vs.~reparameterization and loss functions by likelihood-based vs. quantile-regression-based.}\label{fig:estimators}
\end{figure}

\subsection{Distance Metrics and Loss Functions}

A common perspective in generative modeling is that the choice of model should encode existing metric assumptions about the domain, combined with a generic likelihood-focused loss such as the KL divergence. Under this view, the KL's general applicability and robust optimization properties make it a natural choice, and most implementations of the methods we reviewed in the previous section attempt to, at least indirectly, minimize a version of the KL.

On the other hand, as every model inevitably makes trade-offs when constrained by capacity or limited training, it is desirable for its optimization goal to incentivize trade-offs  prioritizing approximately correct solutions, when the data space is endowed with a metric supporting a meaningful (albeit potentially subjective) notion of approximation. It has been argued \cite{theis2015note,bousquet2017optimal,wgan,bellemare17cramer} that the KL may not always be appropriate from this perspective, by making sub-optimal trade-offs between likelihood and similarity. 

Indeed, many limitations of existing models can be traced back to the use of KL, and the resulting trade-offs in approximate solutions it implies. For instance, its use appears to play a central role in one of the primary failure modes of VAEs, that of blurry samples. \citet{zhao2017towards} argue that the Gaussian posterior $p_\theta(x | z)$ implies an overly simple model, which, when unable to perfectly fit the data, is forced to average (thus creating blur), and is not incentivized by the KL towards an alternative notion of approximate solution. \citet{theis2015note} emphasized that an improvement of log-likelihood does not necessarily translate to higher perceptual quality, and that the KL loss is more likely to produce atypical samples than some other training criteria. 

We offer an alternative perspective: a good model should encode assumptions about the data distribution, whereas a good loss should encode the notion of similarity, that is, the underlying metric on the data space. From this point of view, the KL corresponds to an actual absence of explicit underlying metric, with complete focus on probability.

% (3) Optimal transport, wasserstein... -->
The optimal transport metrics $W_c$, for underlying metric $c(x, x')$, and in particular the $p$-Wasserstein distance, when $c$ is an $L_p$ metric, have frequently been proposed as being well-suited replacements to KL \cite{bousquet2017optimal,genevay2017gan}. Briefly, the advantages are (1) avoidance of mode collapse (no need to choose between spreading over modes or collapsing to a single mode as in KL), and (2) the ability to trade off errors and incentivize approximations that respect the underlying metric. 

% (4) WGAN on the right path, but as has been pointed out recently (CramerGAN and OT paper) the approximations made in solving for $f^*$ can cause arbitrarily bad gradient estimates and solutions.
Recently, \citet{wgan} introduced the Wasserstein GAN, reposing the two-player game as the estimation of the gradient of the $1$-Wasserstein distance between the data and generator distributions. They reframe this in terms of the dual form of the $1$-Wasserstein, with the critic estimating a function $f$ which maximally separates the two distributions. 
While this is an exciting line of work, it still faces limitations when the critic solution is approximate, i.e.~when $f^*$ is not found before each update. In this case, due to insufficient training of the critic \cite{bellemare17cramer} or limitations of the function approximator, the gradient direction produced can be arbitrarily bad \cite{bousquet2017optimal}. 

% (5) How to minimize wasserstein? Recent work QR-DQN and IQN use quantile regression to minimize the Wasserstein metric in 1D. Specifically, we can view QR-DQN and IQN as, in the 1D case, occupying the top right and bottom right columns of the figure respectively.
Thus, we are left with the question of how to minimize a distribution loss respecting an underlying metric. Recent work in distributional reinforcement learning has proposed the use of \textit{quantile regression} as a method for minimizing the $1$-Wasserstein in the univariate case when approximating using a mixture of Dirac functions \cite{dabney2017qr}.

\subsection{Quantile Regression}
In this section, we review quantile regression as a method for estimating the quantile function of a distribution at specific points, i.e.~its inverse cumulative distribution function. This leads to recent work on approximating a distribution by a neural network approximation of its quantile function, acting as a reparameterization of a random sample from the uniform distribution.
 
The quantile regression loss \cite{koenker2001quantile} for a quantile at $\tau \in [0, 1]$ and error $u$ (positive for underestimation and negative for overestimation)
is given by $\rho_\tau(u) = (\tau -  \1\{u\leq 0\}) u $.
It is an asymmetric loss function penalizing underestimation by weight $\tau$ and overestimation by weight $1-\tau$.
For a given scalar distribution $Z$ with c.d.f.~$F_Z$ and a quantile $\tau$, the inverse c.d.f.~$q = F_Z^{-1}(\tau)$ minimizes the expected quantile regression loss $\mathbb{E}_{z\sim Z} \left[ \rho_\tau(z - q)\right]$. 

Using this loss allows one to train a neural network to approximate a scalar distribution represented by its inverse c.d.f. For this, the network can output a fixed grid of quantiles \cite{dabney2017qr}, with the respective quantile regression losses being applied to each output independently.
A more effective approach is to provide the desired quantile $\tau$ as an additional input to the network, and train it to output the corresponding value of $F_Z^{-1}(\tau)$. The \textit{implicit quantile network} (IQN) model \cite{iqn2018} reparameterizes a sample $\tau \sim \U([0,1])$ through a deterministic function to produce samples from the underlying data distribution. These two methods can be seen to belong to the top-right and bottom-right categories in Figure~\ref{fig:estimators}. An IQN $Q_\theta$ can be trained by stochastic gradient descent on the quantile regression loss, with $u = z - Q_\theta(\tau)$ and training samples $(z, \tau)$ drawn from $z \sim Z$ and $\tau \sim \U([0,1])$.

% expand on reasons
One drawback to the quantile regression loss is that gradients do not scale with the magnitude of the error, but instead with the sign of the error and the quantile weight $\tau$. This increases gradient variance and can negatively impact the final model's sample quality. Increasing the batch size, and thus averaging over more values of $\tau$, would have the effect of lowering this variance. Alternatively, we can smooth the gradients as the model converges by allowing errors, under some threshold $\kappa$, to be scaled with their magnitude, reverting to an \textit{expectile} loss. This results in the Huber quantile loss \cite{huber1964robust,dabney2017qr}:
\begin{align}\label{eqn:huberquantile}
    \rho^\kappa_\tau(u) = \begin{cases}
        \frac{|\tau - \1\{u\leq 0\}|}{2\kappa} u^2, \ &\text{if } |u| \le \kappa,\\
        |\tau - \1\{u\leq 0\}|(|u| - \frac{1}{2}\kappa), \ &\text{otherwise}.
    \end{cases}
\end{align}

\section{Autoregressive Implicit Quantiles}
\label{sec:model}

Let $X = (X_1, \ldots, X_n) \in \cX_1 \times \dots \times \cX_n = \cX$ be an $n$-dimensional random variable. We begin by analyzing the effect of two naive applications of IQN to modeling the distribution of $X$.

First, suppose we use the same quantile target, $\tau \in [0,1]$, for every output dimension. The only modification to IQN would be to output $n$ dimensions instead of $1$, the loss being applied to each output dimension independently. This is equivalent to assuming that the dimensions of $X$ are \textit{comonotonic}. Two random variables are comonotonic if and only if they can be expressed as non-decreasing (deterministic) functions of a single random variable \cite{dhaene2006risk}. Thus a joint quantile function for a comonotonic $X$ can be written as $F_X^{-1}(\tau) = (F^{-1}_{X_1}(\tau), F^{-1}_{X_2}(\tau), \ldots, F^{-1}_{X_{n}}(\tau))$. While there are many interesting uses for comonotonic random variables, we believe this assumption is too strong to be useful more broadly.

Second, one could use a separate value $\tau_i \in [0,1]$ for each $X_i$, with the IQN being unchanged from the first case. This corresponds to making an independence assumption on the dimensions of $X$. Again we would expect this to be an unreasonably restrictive modeling assumption for many domains, such as the case of natural images.

Now, we turn to our proposed approach of extending IQN to multivariate distributions. We fix an ordering of the $n$ dimensions. % given by the permutation $\sigma\colon \bN_n \to \bN_n$, which without loss of generality we take to be the identity.
If the density function $p_X$ is expressed as a product of conditional likelihoods, as in Equation~\ref{eqn:conditional_factor}, then the joint c.d.f.~can be written as
\begin{align*}
    F_X(x) &= \operatorname{P}( X_1 \le x_1, \ldots, X_n \le x_n),\\
    &= \prod_{i=1}^n F_{X_{i} | X_{i-1}, \ldots, X_1}(x_i).
\end{align*}
Furthermore, for $\tau_{joint} = \prod_{i=1}^{n} \tau_i$, we can write the \emph{joint-quantile function} of $X$ as
\begin{equation*}
    F^{-1}_X(\tau_{joint}) = (F^{-1}_{X_1}(\tau_1), \ldots, F^{-1}_{X_n | X_{n-1}, \ldots}(\tau_n)).
\end{equation*}
This approach has been used previously by \citet{koenker2006quantile}, who introduced a quantile autoregression model for quantile regression on time-series.

We propose to extend IQN to an autoregressive model of the above conditional form of a joint-quantile function. Denoting 
$\cX_{1:i} = \cX_1 \times \dots \times \cX_i$, let $\tilde \cX := \bigcup_{i=0}^n \cX_{1:i}$ be the space of `partial' data points. We can define the autoregressive IQN as a deterministic function $Q_\theta\colon \tilde \cX \times [0, 1]^n \to \tilde \cX$, mapping partial samples $\tilde x \in \tilde \cX$ and quantile targets $\tau_i \in [0, 1]$ to estimates of $F^{-1}_X$. We can then train $Q_\theta$ using a quantile regression loss (Equation~\ref{eqn:huberquantile}). For generation, one can iterate $x_{1:i} = Q_\theta(x_{1:i-1}, \tau_{i})$, on a sequence of growing partial samples\footnote{Throughout we understand $x_0 = x_{1:0} \in \cX_{1:0}$ to denote the `empty tuple', and the function $Q_\theta$ to map this to a single unconditional sample $x_1 = x_{1:1} = Q_\theta(x_0, \tau_1)$.} $x_{1:i-1}$ and independently sampled $\tau_{i} \sim \U([0,1])$, for $i = 1, \ldots, n$, to finally obtain a sample $x = x_{1:n}$.

% Our method has similarities to generative moment matching networks (GMMN) \cite{li2015generative}. However, the Maximum Mean Discrepancy objective relies on mini-batches large enough to be a close approximation to the data distribution, whereas our quantile-regression loss remains unbiased even for single sample updates. Additionally, GMMN rely on predefined kernel function and a bandwidth hyperparameter, while our work requires no additional hyperparameters.

\subsection{Quantile Regression and the Wasserstein}

As previously mentioned, for the restricted model class of a uniform mixture of Diracs, quantile regression can be shown to minimize the $1$-Wasserstein metric \cite{dabney2017qr}. We extend this analysis for the case of arbitrary approximate quantile functions, and find that quantile regression minimizes a closely related divergence which we call {\em quantile divergence}, defined, for any distributions $P$ and $Q$, as
\begin{equation*}
    q(P, Q) := \int_0^1 \left[ \int_{F_P^{-1}(\tau)}^{F_{Q}^{-1}(\tau)} (F_P(x) - \tau) dx \right] d\tau.
\end{equation*}

% The quantile divergence *is* a statistical divergence. We can see that it 
% satisfies the non-negative requirement due to monotonicity of the quantile 
% function and cdf combined with the identity that switches signs when 
% F_Q^{-1} < F_P^{-1}. 

Indeed, the expected quantile loss of any parameterized quantile function $\bar Q_\theta$ equals, up to a constant, the quantile divergence between $P$ and the distribution $Q_\theta$ implicitly defined by $\bar Q_\theta$:
$$ \expect_{\tau\sim\U([0,1])} \big[ \expect_{z\sim P} [\rho_\tau (z- \bar Q_\theta(\tau))]\big] = q(P, Q_\theta) + h(P),$$
where $h(P)$ does not depend on $Q_\theta$. Thus quantile regression minimizes the quantile divergence $q(P, Q_\theta)$ and the sample gradient $\nabla_\theta\rho_\tau(z- \bar Q_\theta(\tau))$ (for $\tau\sim\U([0,1])$ and $z\sim P$) is an unbiased estimate of $\nabla_\theta q(P, Q_\theta)$. See Appendix for proofs. 

\subsection{Quantile Density Function}
% Likelihoods
Although IQN does not directly model the log-likelihood of the data distribution, observe that we can still query the implied density at a point \cite{jones1992estimating}:
\begin{equation*}
    \frac{\partial}{\partial \tau}F_X^{-1}(\tau) = \frac{1}{p_X(F_X^{-1}(\tau))}.
\end{equation*}
Indeed, this quantity, known as the \textit{sparsity function} \cite{tukey1965part} or the \textit{quantile-density function} \cite{parzen1979nonparametric} plays a central role in the analysis of quantile regression models \cite{koenker1994confidence}. A common approach involves choosing a bandwidth parameter $h$ and estimating this quantity through finite-differences around the value of interest as $(F_X^{-1}(\tau + h) - F_X^{-1}(\tau - h))/2h $ \cite{siddiqui1960distribution}. However, as we have the full quantile function, the quantile-density function can be computed exactly using a single step of back-propagation to compute $\frac{\partial F^{-1}(\tau)}{\partial\tau}$. As this only allows querying the density given the value of $\tau$, application to general likelihoods would require finding the value of $\tau$ that produces the closest approximation to the query point $x$. Though arguably too inefficient for training, this could potentially be used to interrogate the model.

\begin{figure}[t]
\begin{center}
\includegraphics[width=.8\textwidth]{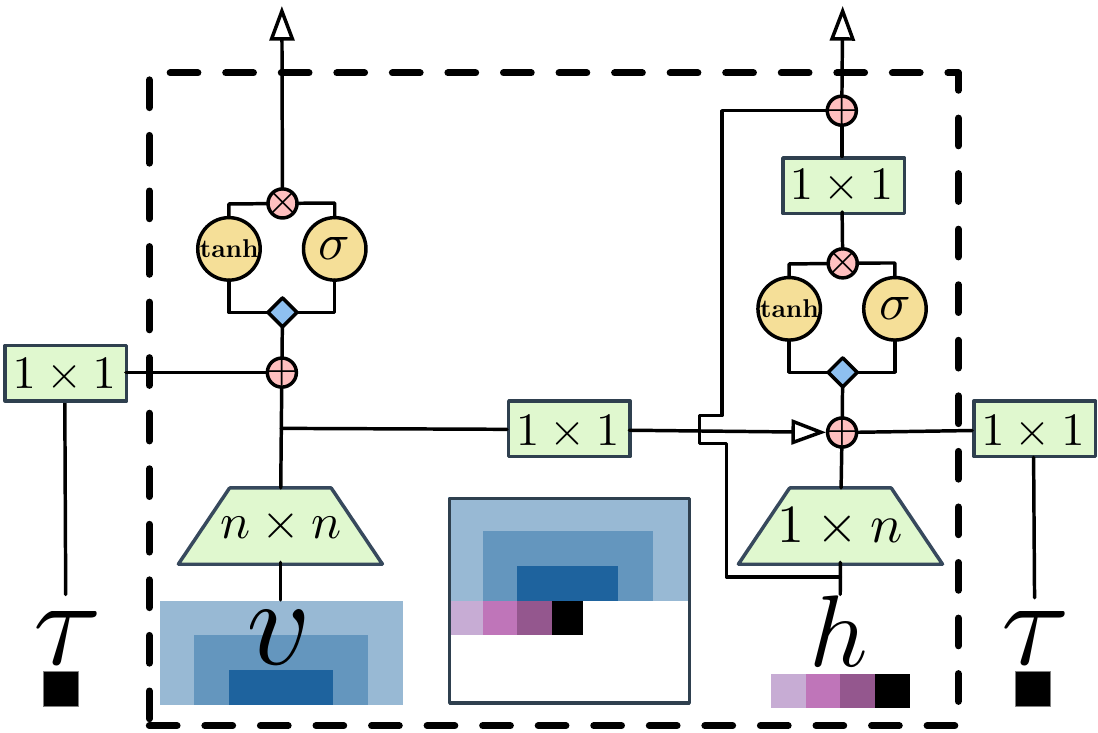}
\end{center}
\caption{Illustration of Gated PixelCNN layer block for PixelIQN. Dashed line shows boundary of standard Gated PixelCNN, with $v$ the vertical and $h$ the horizontal stack. Conditioning on $\tau$ is identical to the location-dependent conditioning in Gated PixelCNN.}\label{fig:pixeliqn}
\end{figure}

\section{PixelIQN}
\label{sec:pixeliqn}
% Specifics and results 

\begin{figure*}[t]
\begin{center}
\includegraphics[width=0.33\textwidth]{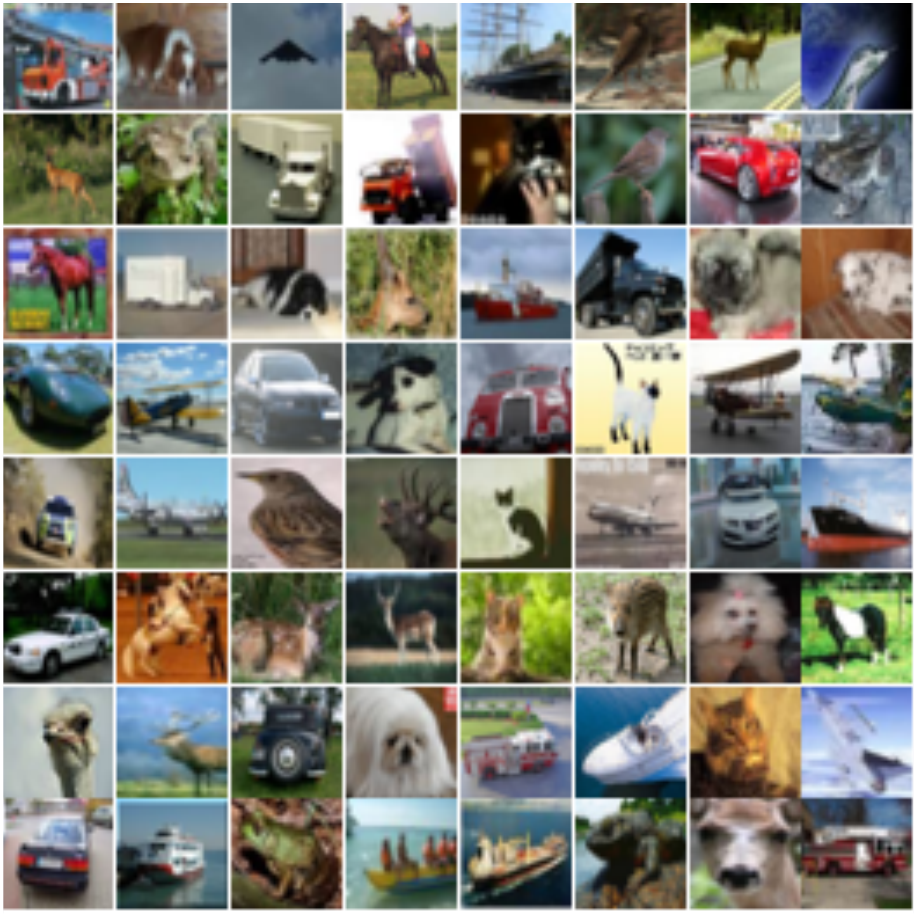}
\includegraphics[width=0.33\textwidth]{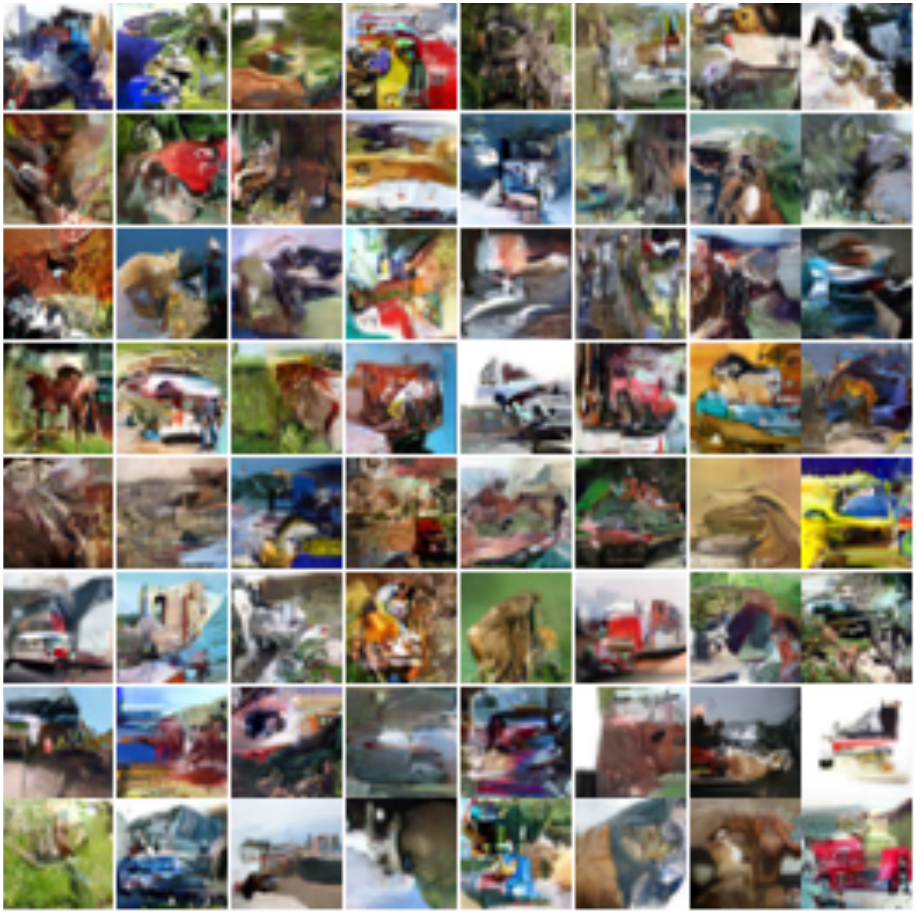}
\includegraphics[width=0.33\textwidth]{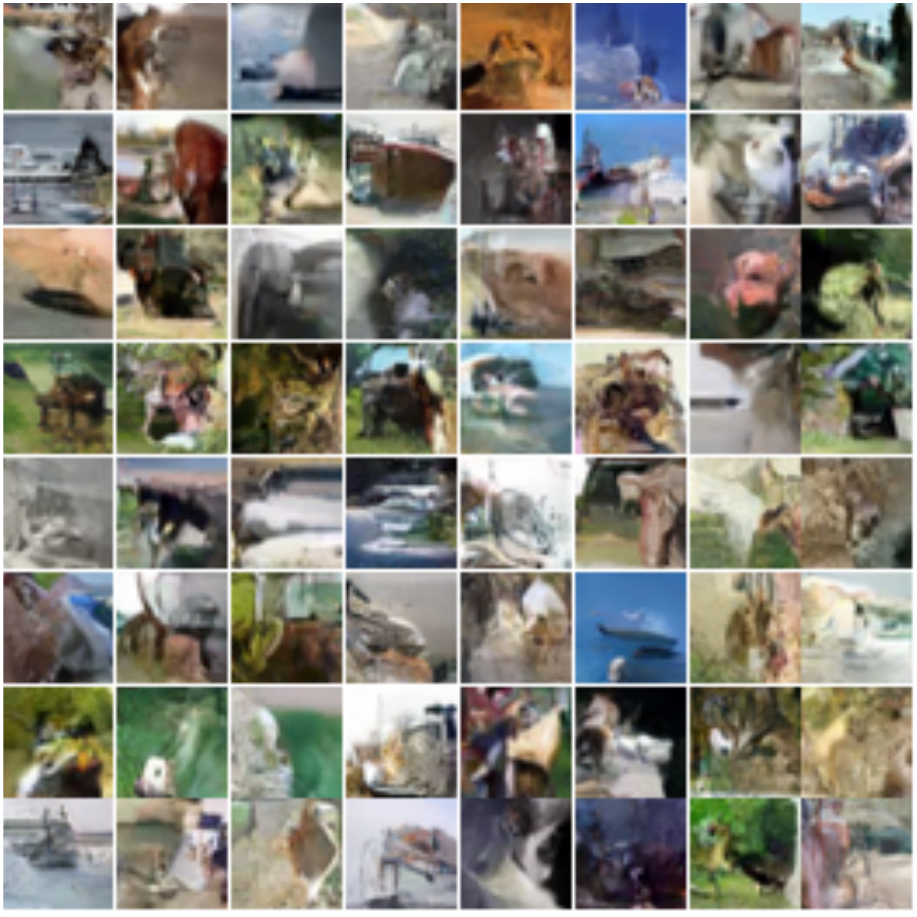}
\end{center}
\caption{CIFAR-10: Real example images (left), samples generated by PixelCNN (center), and samples generated by PixelIQN (right).}\label{fig:cifar_samples}
\end{figure*}

To test our proposed method, which is architecturally compatible with many generative model approaches, we wanted to compare and contrast IQN, that is quantile regression and quantile reparameterization, with a method trained with an explicit parameterization to minimize KL divergence. A natural choice for this was PixelCNN, specifically we build upon the Gated PixelCNN of \citet{van2016conditional}.

The Gated PixelCNN takes as input an image $x \sim X$, sampled from the training distribution at training time, and potentially all zeros or partially generated at generation time, as well as a location-dependent context $s$. The model consists of a number of residual layer blocks, whose structure is chosen to allow each output pixel to be a function of all preceding input pixels (in a raster-scan order).
At its core, each layer block computes two gated activations of the form
\begin{equation*}
    y = \tanh(W_{k,f} * x + V_{k,f} * s) \odot \sigma(W_{k,g} * x + V_{k,g} * s),
\end{equation*}
with $k$ the layer index, $*$ denoting convolution, and $V_{k,f}$ and $V_{k,g}$ being $1\times1$ convolution kernels. See Figure~\ref{fig:pixeliqn} for a full schematic depiction of a Gated PixelCNN layer block.
After a number of such layer blocks, the PixelCNN produces a final output layer with shape $(n, n, 3, 256)$, with a softmax across the final dimension, corresponding to the approximate conditional likelihood for the value of each pixel-channel. That is, the conditional likelihood is the product of these individual autoregressive models,
\begin{equation*}
    p(x | s) = \prod_{i=1}^{3n^2} p(x_i | x_1, \ldots, x_{i-1}, s_i).
\end{equation*}

Typically the location-dependent conditioning term was used to condition on class labels, but here, we will use it to condition on the sample point\footnote{Conditioning on labels remains possible (see Section \ref{sec:imagenet}).} $\tau \in [0, 1]^{3n^2}$.
Thus, in addition to the input image $x$ we input, in place of $s$, the sample points $\tau = (\tau_1, \ldots, \tau_{3n^2})$ to be reparameterized, with each $\tau_i \sim \U([0,1])$. Finally, our network outputs only the full sample image of shape $(n, n, 3)$, without the need for an additional softmax layer. Note that the number of $\tau$ values generated exactly corresponds to the number of random draws from softmax distributions in the original PixelCNN. We are simply changing the role of the randomness, from a draw at the output to a part of the input.

Architecturally, our proposed model, PixelIQN, is exactly the network given by \citet{van2016conditional}, with the one exception that we output only a single value per pixel-channel and do not require the softmax activations.

In PixelCNN training is done by passing the training image through the network, and training each output softmax distribution using the KL divergence between the training image and the approximate distribution,
\begin{equation*}
    \sum_i D_{KL}(\delta_{x_i}, p(\cdot | x_1, \ldots, x_{i-1})).
\end{equation*}
For PixelIQN, the input is the training image $x$ and a sample point $\tau \sim\U([0,1]^{3n^2})$. The output values $Q_x(\tau) \in \bR^{3n^2}$ are interpreted as the approximate quantile function at $\tau$, $Q_x(\tau)_i = Q_X(\tau_i | x_{i-1}, \ldots)$, trained with a single step of quantile regression towards the observed sample $x$:
\begin{equation*}
    \sum_i \rho_{\tau_i}^{\kappa}(x_i - Q_X(\tau_i | x_{i-1}, \ldots)).
\end{equation*}

%$$F^{-1}_X(\tau) = (F^{-1}_{X_1}(\tau_1}, F^{-1}_{X_2}(\tau_2}, \ldots, F^{-1}_{X_n}(\tau_n)$$
%$$Q_X(\tau) = (Q_{X_1}(\tau_1), Q_{X_2}(\tau_2), \ldots, Q_{X_n}(\tau_n))$$
% 

\subsection{CIFAR-10}
\label{sec:cifar}

\begin{figure*}[t]
\begin{floatrow}
\ffigbox{
\includegraphics[width=.5\textwidth]{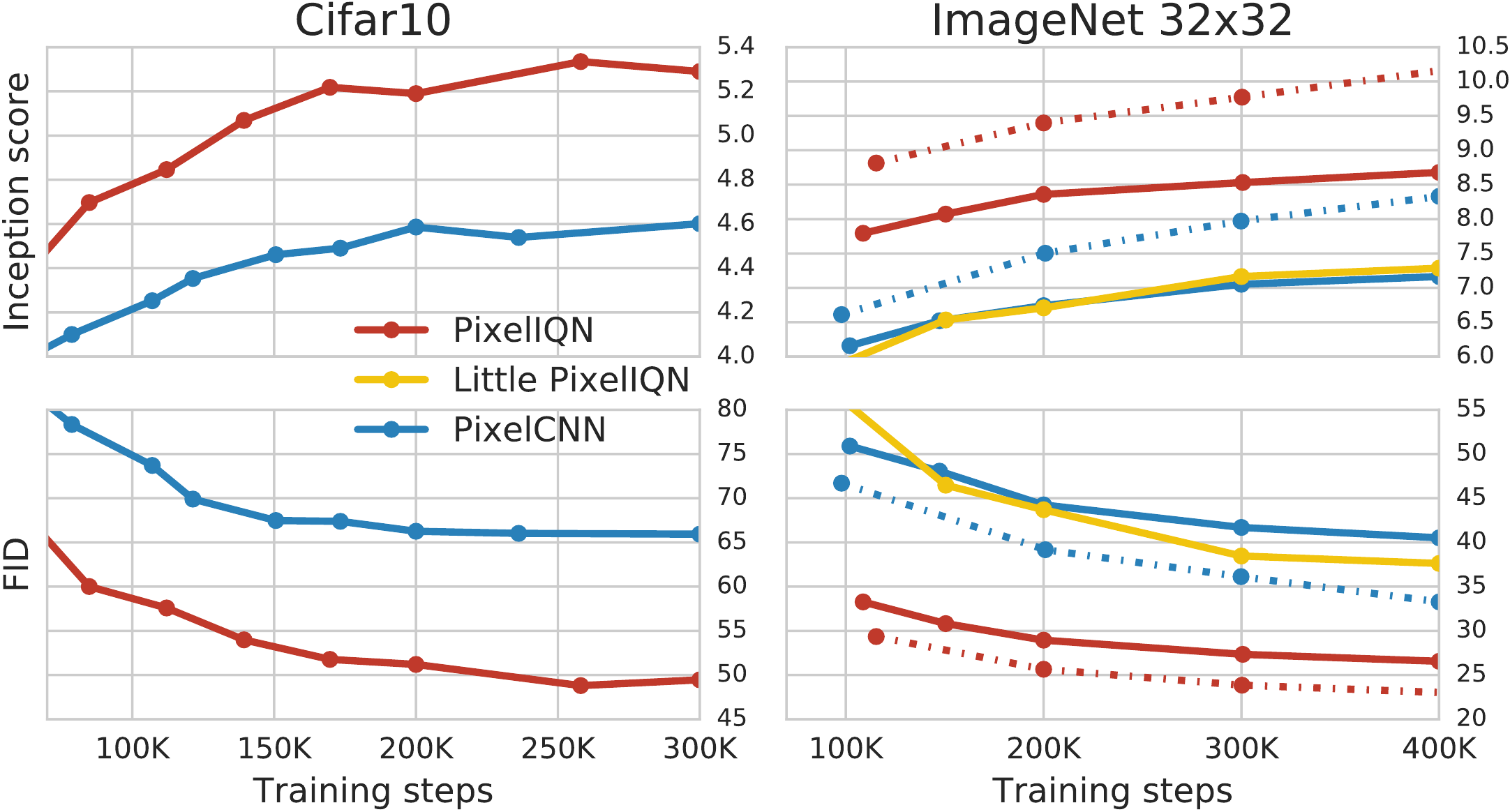}
}{
  %\vspace{-.5cm}
  \caption{Evaluations by Inception score (higher is better) and FID (lower is better) on CIFAR-10 and ImageNet 32x32. Dotted lines correspond to models trained with class-label conditioning.}\label{fig:fid_scores}
}
\capbtabbox{
\begin{tabular}{l | l | l | l | l}
         & \multicolumn{2}{c}{\underline{CIFAR-10}} & \multicolumn{2}{c}{\underline{ImageNet (32x32)}} \\
Method   & \small{Inception}       & \textsc{fid}        & \small{Inception}        & \textsc{fid}          \\
\hline
\hline
WGAN     & 3.82            & -          & -                & -            \\
WGAN-GP  & 6.5             & 36.4       & -                & -            \\
DC-GAN   & 6.4             & 37.11      & 7.89             & -            \\
PixelCNN & 4.60            & 65.93      & 7.16             & 40.51        \\
PixelIQN & 5.29            & 49.46      & 8.68             & 26.56        \\
PixelIQN(l) & - & - & 7.29 & 37.62        \\
\hline
PixelCNN$^*$ & - & - & 8.33 & 33.27 \\
PixelIQN$^*$ & - & - & 10.18 & 22.99 \\
\end{tabular}
}{
  %\vspace{-.5cm}
  \caption{Inception score and FID for CIFAR-10 and ImageNet. WGAN and DC-GAN results taken from \cite{wgan, radford2015unsupervised}. PixelIQN(l) is the small 15-layer version of the model. Models marked $*$ refer to class-conditional training.}
  \label{tab:scores}
}
\end{floatrow}
\end{figure*}

We begin by demonstrating PixelIQN on CIFAR-10 \cite{krizhevsky2009learning}. For comparison, we train both a baseline Gated PixelCNN and a PixelIQN. Both models correspond to the $15$-layer network variant in \cite{van2016conditional}, see Appendix for detailed hyperparameters and training procedure. The two methods have substantially different loss functions, so we performed a hyperparameter search using a short training run, with the same number ($500$) of hyperparameter configurations evaluated for both models. For all results, we report full training runs using the best found hyperparameters in each case. The evaluation metric used for the hyperparameter search was the Fr\'echet Inception Distance (FID) \cite{heusel2017gans}, see Appendix for details.
In addition to FID, we report Inception score \cite{salimans2016improved} for both models.

Figure~\ref{fig:fid_scores} (left) shows Inception score and FID for both models evaluated at several points throughout training. The fully trained PixelCNN achieves an Inception score and FID of $4.6$ and $65.9$ respectively, while PixelIQN substantially outperforms it with an Inception score of $5.3$ and FID of $49.5$. This also compares favorably with e.g.~WGAN \cite{wgan}, which reaches an Inception score of $3.8$. For subjective evaluations, we give samples from both models in Figure~\ref{fig:cifar_samples}. Samples coming from PixelIQN are much more visually coherent. Of note, the PixelIQN model achieves a performance level comparable to that of the fully trained PixelCNN with only about one third the number of training updates (and about one third of the wall-clock time).

\subsection{ImageNet 32x32}
\label{sec:imagenet}

\begin{figure*}[t]
\begin{center}
\includegraphics[width=0.33\textwidth]{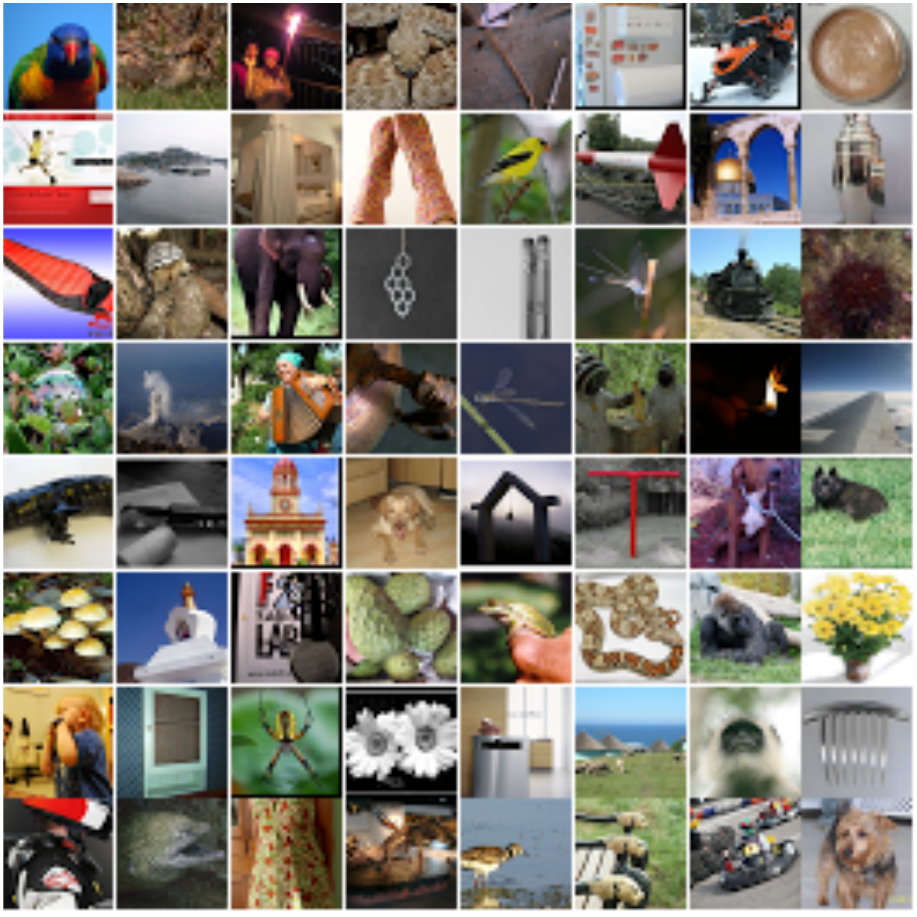}
\includegraphics[width=0.33\textwidth]{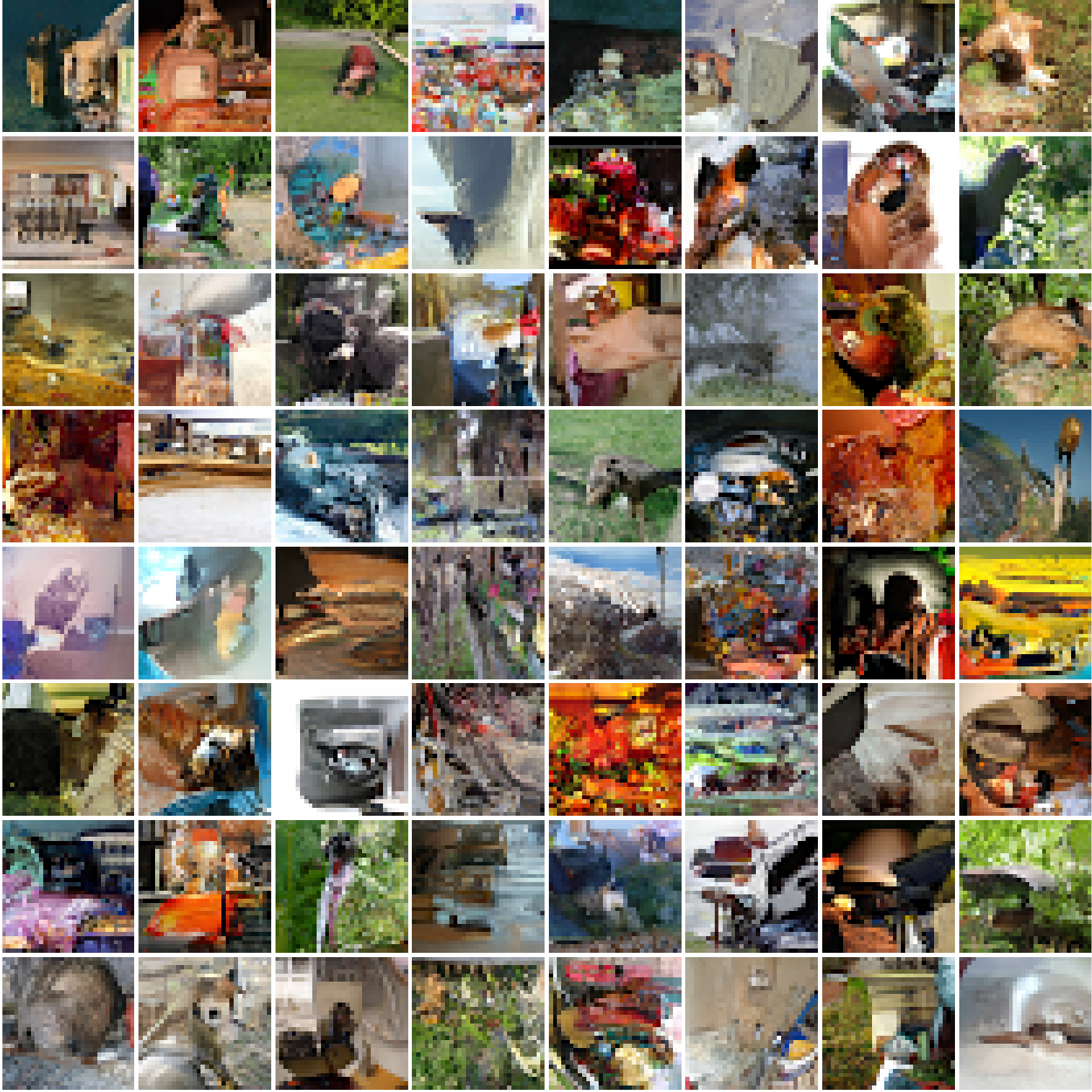}
\includegraphics[width=0.33\textwidth]{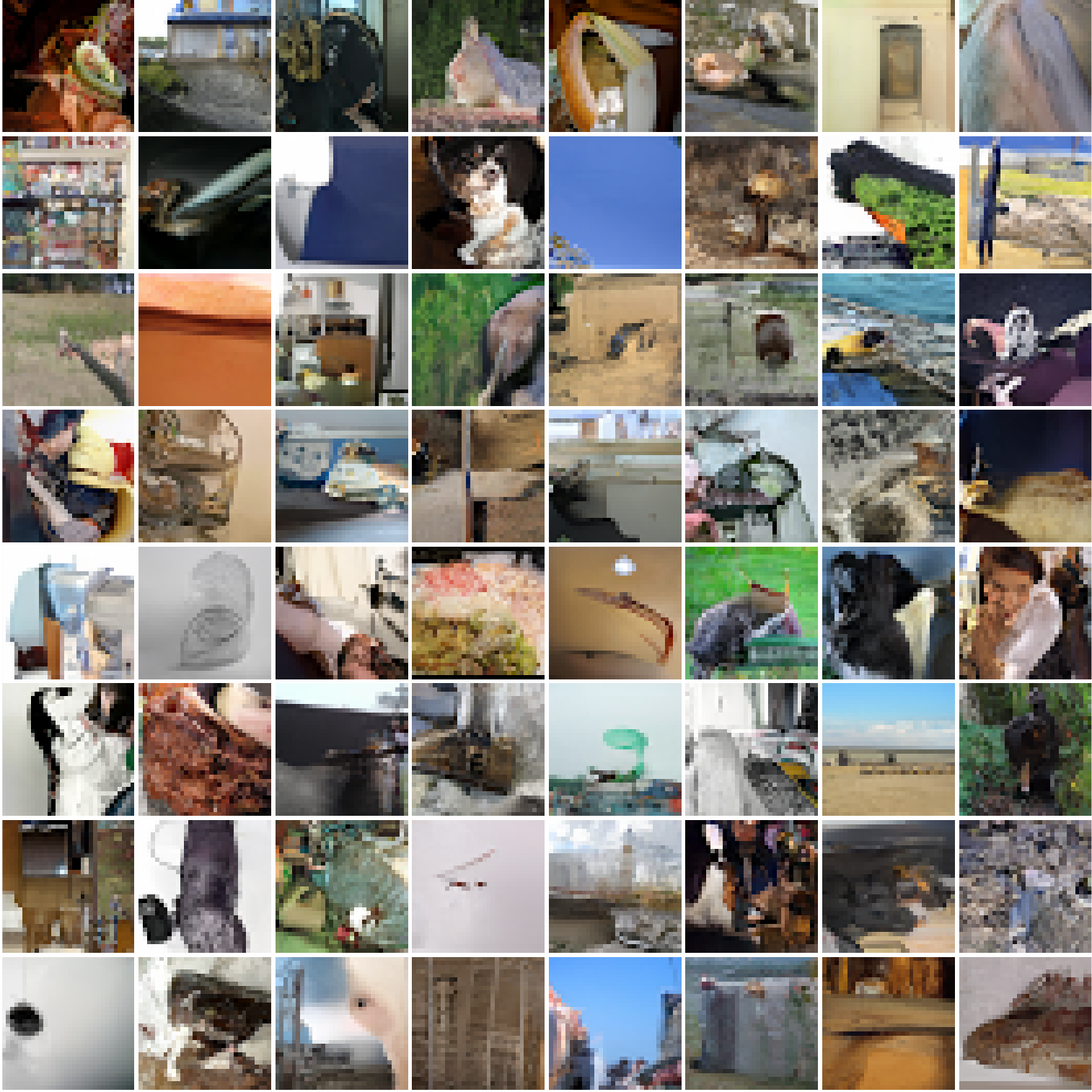}
\end{center}
\caption{ImageNet 32x32: Real example images (left), samples generated by PixelCNN (center), and samples generated by PixelIQN (right). Neither of the sampled image sets were cherry-picked. More samples by PixelIQN in the Appendix.}\label{fig:imagenet_samples}
\end{figure*}

Next, we turn to the small ImageNet dataset \cite{russakovsky2015imagenet}, first used for generative modeling in the PixelRNN work \cite{vandenoord16pixel}. Again, we evaluate using FID and Inception score. For this much harder dataset, we base our PixelCNN and PixelIQN models on the larger $20$-layer variant used in \cite{van2016conditional}. Due to substantially longer training time for this model, we did not perform additional hyperparameter tuning, and mostly used the same hyperparameter values as in the previous sections for both models; details can be found in the Appendix.

Figure~\ref{fig:fid_scores} shows Inception score and FID throughout training of PixelCNN and PixelIQN. Again, PixelIQN substantially outperforms the baseline in terms of final performance and sample complexity. For final scores and a comparison to state-of-the-art GAN models, see Table~\ref{tab:scores}. Figure~\ref{fig:imagenet_samples} shows random (non-cherry-picked) samples from both models. Compared to PixelCNN, PixelIQN samples appear to have superior quality with more global consistency and less `high-frequency noise'.

In Figure~\ref{fig:imagenet_inpaint}, we show the inpainting performance of PixelIQN, by fixing the top half of a validation set image as input and sampling repeatedly from the model to generate different completions. We note that the model consistently generates plausible completions with significant diversity between different completion samples for the same input image. Meanwhile, WGAN-GP has been seen to produce deterministic completions \cite{bellemare17cramer}.

Following \cite{van2016conditional}, we also trained a class-conditional PixelIQN variant, providing to the model the one-hot class label corresponding to a training image (in addition to a $\tau$ sample). Samples from a class-conditional model can be expected to have higher visual quality, as the class label provides $\log_2(1000) \approx 10$ bits of information, see Figure~\ref{fig:class_conditional}. As seen in Figure~\ref{fig:fid_scores} and Table~\ref{tab:scores}, class conditioning also further improves Inception score and FID. To generate each sample for the computation of these scores, we sample one of 1000 class labels randomly, then generate an image conditioned on this label via the trained model.

Finally, motivated by the very long training time for the large PixelCNN model (approximately 1 day per 100K training steps, on 16 NVIDIA Tesla P100 GPUs), we also trained smaller $15$-layer versions of the models (same as the ones used on CIFAR-10) on the small ImageNet dataset. For comparison, these take approximately 12 hours for 100K training steps on a single P100 GPU, or less than 3 hours on 8 P100 GPUs. As expected, little PixelCNN, while suitable for the CIFAR-10 dataset, fails to achieve competitive scores on the ImageNet dataset, achieving Inception score $5.1$ and FID $66.4$. Astonishingly, little PixelIQN on this dataset reaches Inception score $7.3$ and FID $38.5$, see Figure~\ref{fig:fid_scores} (right). It thereby not only outperforms the little PixelCNN, but also the larger $20$-layer version! This strongly supports the hypothesis that PixelCNN, and potentially many other models, are constrained not only by their model capacity, but crucially also by the sub-optimal trade-offs made by their log-likelihood training criterion, failing to align with perceptual or evaluation metrics.

\begin{figure}[t]
\begin{center}
\includegraphics[width=.8\textwidth]{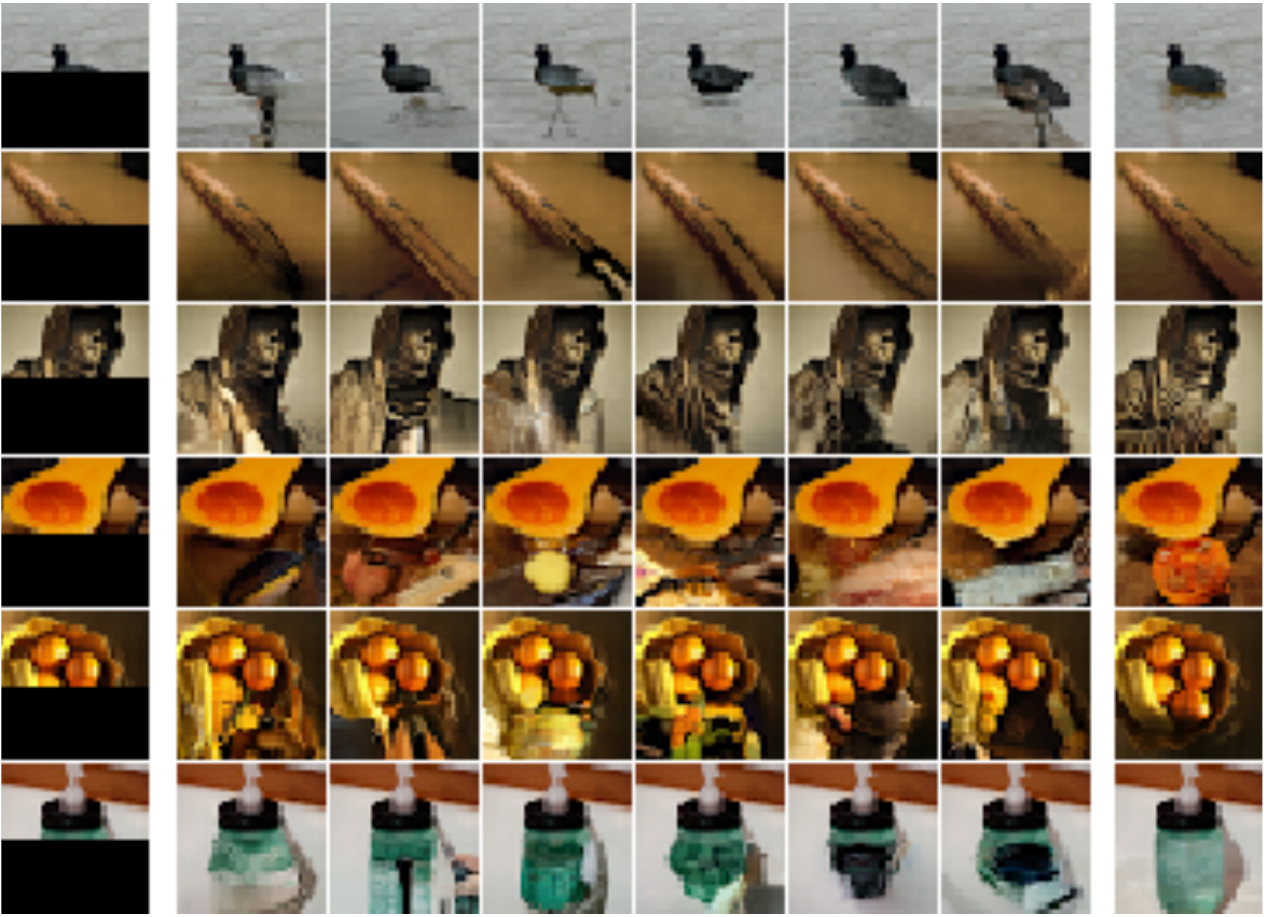}
\end{center}
\caption{Small ImageNet inpainting examples. Left image is the input provided to the network at the beginning of sampling, right is the original image, columns in between show different completions. More examples in the Appendix.}\label{fig:imagenet_inpaint}
\end{figure}

\begin{figure}[t]
\begin{center}
\includegraphics[width=.9\textwidth]{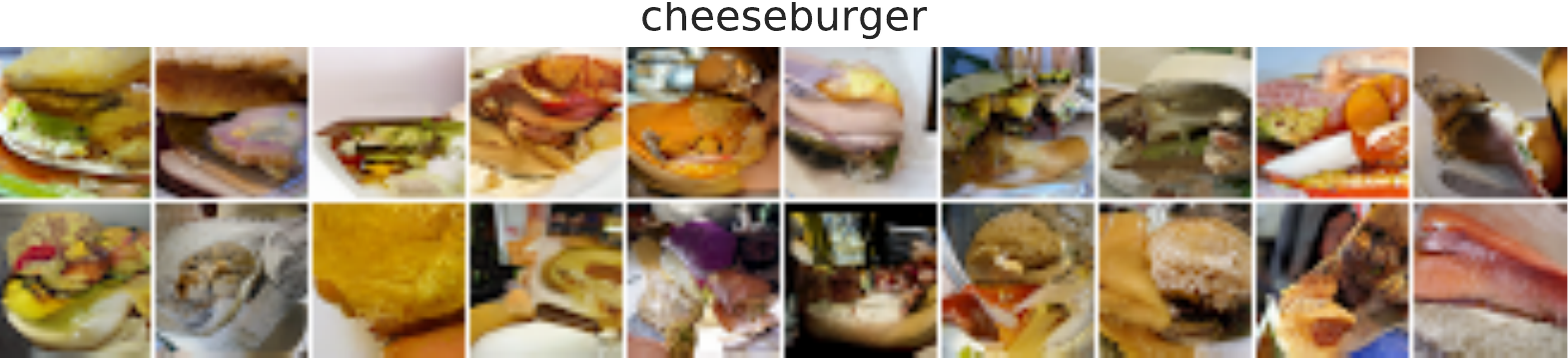}
\includegraphics[width=.9\textwidth]{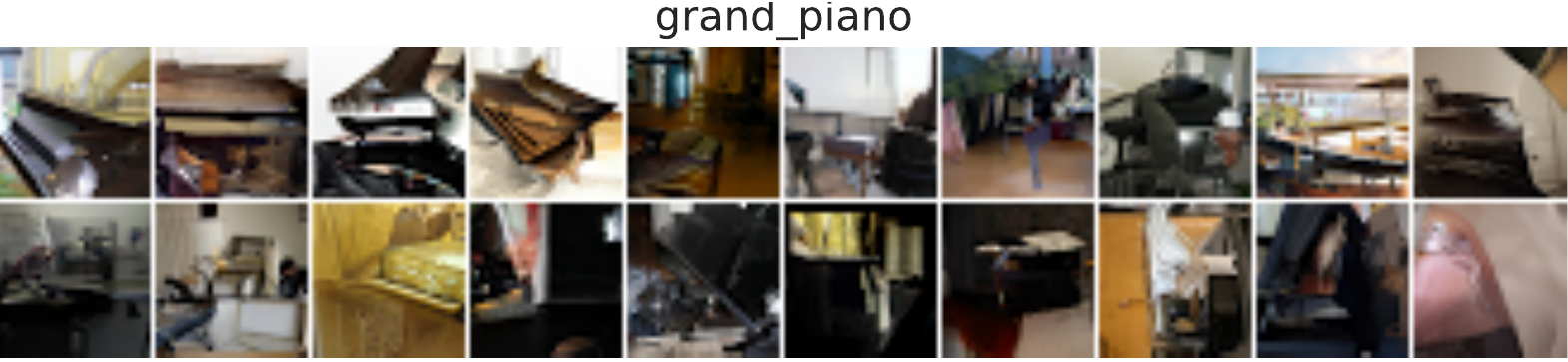}
\includegraphics[width=.9\textwidth]{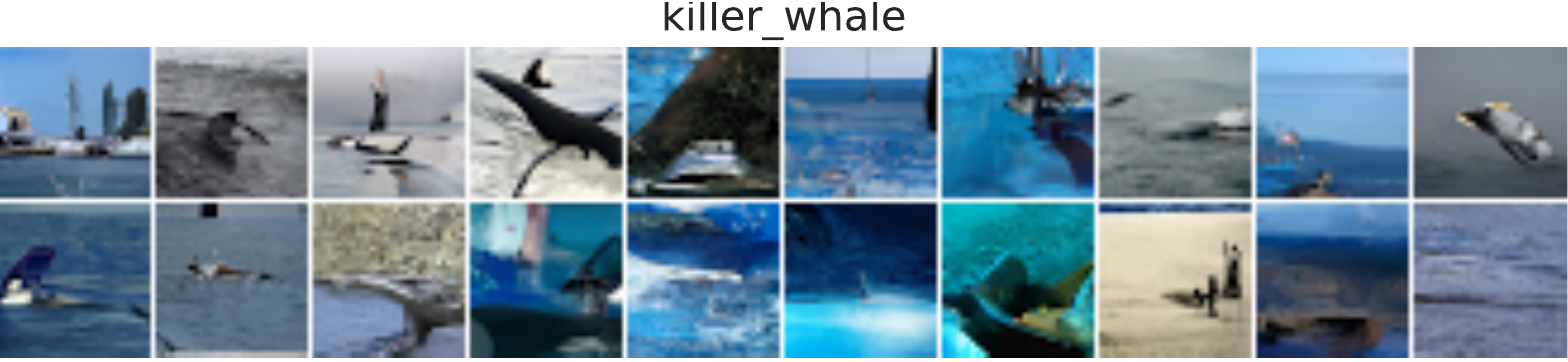}
\includegraphics[width=.9\textwidth]{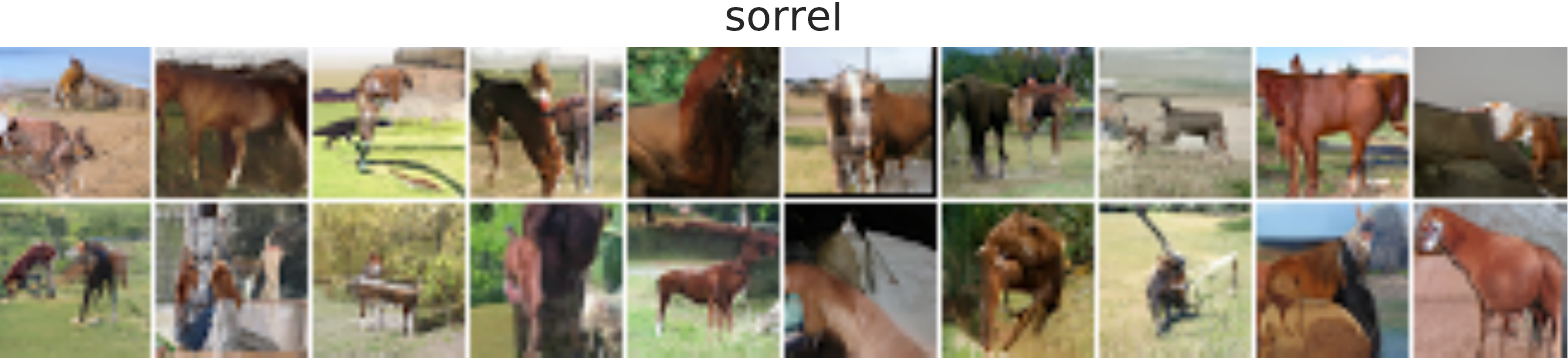}
\includegraphics[width=.9\textwidth]{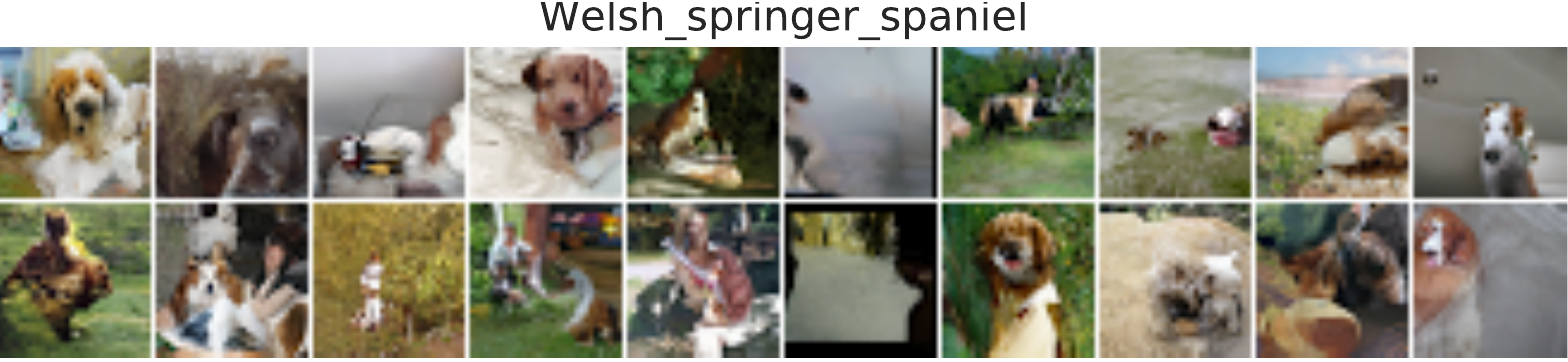}
\end{center}
\caption{Class-conditional samples from PixelIQN. More samples of each class and more classes in the Appendix.}\label{fig:class_conditional}
\end{figure}

\section{Discussion and Conclusions}
\label{sec:conclusion}

% novel approach to gen modeling
Most existing generative models for images belong to one of two classes. The first are likelihood-based models, trained with an elementwise KL reconstruction loss, which, while perceptually meaningless, provides robust optimization properties and high sample diversity.
The second are GANs, trained based on a discriminator loss, typically better aligned with a perceptual metric and enabling the generator to produce realistic, globally consistent samples. Their advantages come at the cost of a harder optimization problem, high parameter sensitivity, and most importantly, a tendency to collapse modes of the data distribution.

% advantages
AIQNs are a new, fundamentally different, technique for generative modeling. By using a quantile regression loss instead of KL divergence, they combine some of the best properties of the two model classes. By their nature, they preserve modes of the learned distribution, while producing perceptually appealing high-quality samples. 
% Moreover, empirically they exceed even their likelihood-based counterparts in terms of robustness to hyperparameter choice.
The inevitable approximation trade-offs a generative model makes when constrained by capacity or insufficient training can vary significantly depending on the loss used. We argue that the proposed quantile regression loss aligns more effectively with a given metric and therefore makes subjectively more advantageous trade-offs.

Devising methods for quantile regression over multidimensional outputs is an active area of research. New methods are continuing to be investigated \cite{carlier2016,hallin2016multiple}, and a promising direction for future work is to find ways to use these to replace autoregressive models. One approach to reducing the computational burden of such models is to apply AIQN to the latent dimensions of a VAE. Similar in spirit to \citet{rosca2017variational}, this would use the VAE to reduce the dimensionality of the problem and the AIQN to sample from the true latent distribution. In the Appendix we give preliminary results using such an technique, on CelebA $64 \times 64$ \cite{liu2015faceattributes}.

% applicability & simplicity, empirical advantages
We have shown that IQN, computationally cheap and technically simple, can be readily applied to existing architectures, PixelCNN and VAE (Appendix), improving robustness and sampling quality of the underlying model. We demonstrated that PixelIQN produces more realistic, globally coherent samples, and improves Inception score and FID. 

We further point out that many recent advances in generative models could be easily combined with our proposed method. Recent algorithmic improvements to GANs such as mini-batch discrimination and progressive growing \cite{salimans2016improved,karras2017progressive}, while not strictly necessary in our work, could be applied to further improve performance.
PixelCNN++ \cite{salimans2017pcnn} is an architectural improvement of PixelCNN, with several beneficial modifications supported by experimental evidence. Although we have built upon the original Gated PixelCNN in this work, we believe all of these modifications to be compatible with our work, except for the use of a mixture of logistics in place of PixelCNN's softmax. As we have entirely replaced this model component, this change does not map onto our model. Of note, the motivation behind this change closely mirrors our own, in looking for a loss that respects the underlying metric between examples. The recent PixelSNAIL model \cite{chen2017pixelsnail} achieves state-of-the-art modeling performance by enhancing PixelCNN with ELU nonlinearities, modified block structure, and an attention mechanism. Again, all of these are fully compatible with our work and should improve results further.

Finally, the implicit quantile formulation lifts a number of architectural restrictions of previous generative models. Most importantly, the reparameterization as an inverse c.d.f.~allows to learn distributions over continuous ranges without pre-specified boundaries or quantization. This enables modeling continuous-valued variables, for example for generation of sound \cite{van2016wavenet}, opening multiple interesting avenues for further investigation.

\section*{Acknowledgements}
We would like to acknowledge the important role many of our colleagues at DeepMind played for this work.
We especially thank A\"aron van den Oord and Sander Dieleman for invaluable advice on the PixelCNN model; Ivo Danihelka and Danilo J. Rezende for careful reading and insightful comments on an earlier version of the paper; Igor Babuschkin, Alexandre Galashov, Dominik Grewe, Jacob Menick, and Mihaela Rosca for technical help.

\bibliography{distrl}
\bibliographystyle{icml2018}

\clearpage

\section*{Appendix}

\section*{Quantile regression minimizes the quantile divergence}
% \label{sec:quantile.distance}
%\zdistance*

\begin{proposition}
For any distributions $P$ and $Q$, define the quantile divergence
\begin{equation*}
    q(P, Q) := \int_0^1 \left[ \int_{F_P^{-1}(\tau)}^{F_{Q}^{-1}(\tau)} (F_P(x) - \tau) \,dx \right] d\tau.
\end{equation*}
Then the expected quantile loss of a quantile function $\bar Q$ implicitly defining the distribution $Q$ satisfies 
$$ \expect_{\tau\sim\U([0,1])} \expect_{X\sim P} \big[\rho_\tau (X- \bar Q(\tau))\big] = q(P, Q) + h(P),$$
where $h(P)$ does not depend on $Q$. 
\end{proposition}

\begin{proof}

Let $P$ be a distribution with p.d.f.~$f_P$ and c.d.f.~$F_P$. Define
\begin{align*}
    \rho_\tau(u) &= u (\tau - \1\{u\leq 0\}),\\
    g_\tau(q) &= \expect_{X\sim P}[\rho_\tau(X-q)].
\end{align*}
We have, for any $q$ and $\tau$,
\beqan
g_\tau(q) &=& \int_{-\infty}^q (x-q)(\tau-1)f_P(x) \,dx \\
& &\quad + \int_q^{\infty} (x-q) \tau f_P(x) \,dx \\
&=& \int_{-\infty}^q (q-x) f_P(x) dx + \int_{-\infty}^{\infty} (x-q) \tau f_P(x) \,dx \\
&=& q F_P(q) + \int_{-\infty}^q F_P(x) \,dx - [ x F_P(x)]_{-\infty}^q \\
& & \quad +\ \tau \left(\expect_{X\sim P}[X] - q\right) \\
&=& \int_{-\infty}^q F_P(x) \,dx + \tau \left(\expect_{X\sim P}[X] - q\right),
\eeqan 
where the third equality follows from an integration by parts of $\int_{-\infty}^q x f_P(x)\,dx$. Thus the function $q\mapsto g_\tau(q)$ is minimized for $q=F^{-1}_P(\tau)$ and its minimum is
\beqan
g_\tau(F^{-1}_P(\tau)) = \int_{-\infty}^{F^{-1}_P(\tau)} F_P(x) \,dx + \tau \left(\expect_{X\sim P}[X] - F^{-1}_P(\tau)\right).
\eeqan 
We deduce that
\beqan
&& g_\tau(q) - g_\tau(F^{-1}_P(\tau)) \\
&=& \int_{F^{-1}_P(\tau)}^q F_P(x) \,dx + \tau (F^{-1}_P(\tau) - q) \\
&=& \int_{F^{-1}_P(\tau)}^q (F_P(x)-\tau) \,dx.
\eeqan
Thus for a quantile function $\bar Q$, we have the expected quantile loss:
$$\expect_{\tau\sim\U([0,1])} \big[g_\tau(Q(\tau)) \big]= q(P,Q) + \underbrace{\expect_{\tau\sim\U([0,1])} \big[ g_\tau(F^{-1}_P(\tau))\big]}_{\mbox{does not depend on }Q}.$$
This finishes the proof of the proposition.
\end{proof}

We observe that quantile regression is nothing else than a projection under the quantile divergence. Thus for a parametrized quantile function $\bar Q_\theta$ with corresponding distribution $Q_\theta$, the sample-based quantile regression gradient $\nabla_\theta \rho_{\tau}(X - \bar Q_\theta(\tau))$ for a sample $\tau\sim\U([0,1])$ and $X\sim P$ is an unbiased estimate of $\nabla_\theta q(P,Q_\theta)$:
\begin{align*}
    \expect\big[ \nabla_\theta  \rho_{\tau}(X -\bar Q_\theta(\tau)) \big] &= \nabla_\theta \expect_{\tau\sim\U([0,1])} \big[g_\tau(\bar Q_\theta(\tau)) \big],\\
    &= \nabla_\theta q(P, Q_\theta).
\end{align*}

\begin{figure}[t]
\begin{center}
\includegraphics[width=.8\textwidth,trim=0 1cm 0 0]{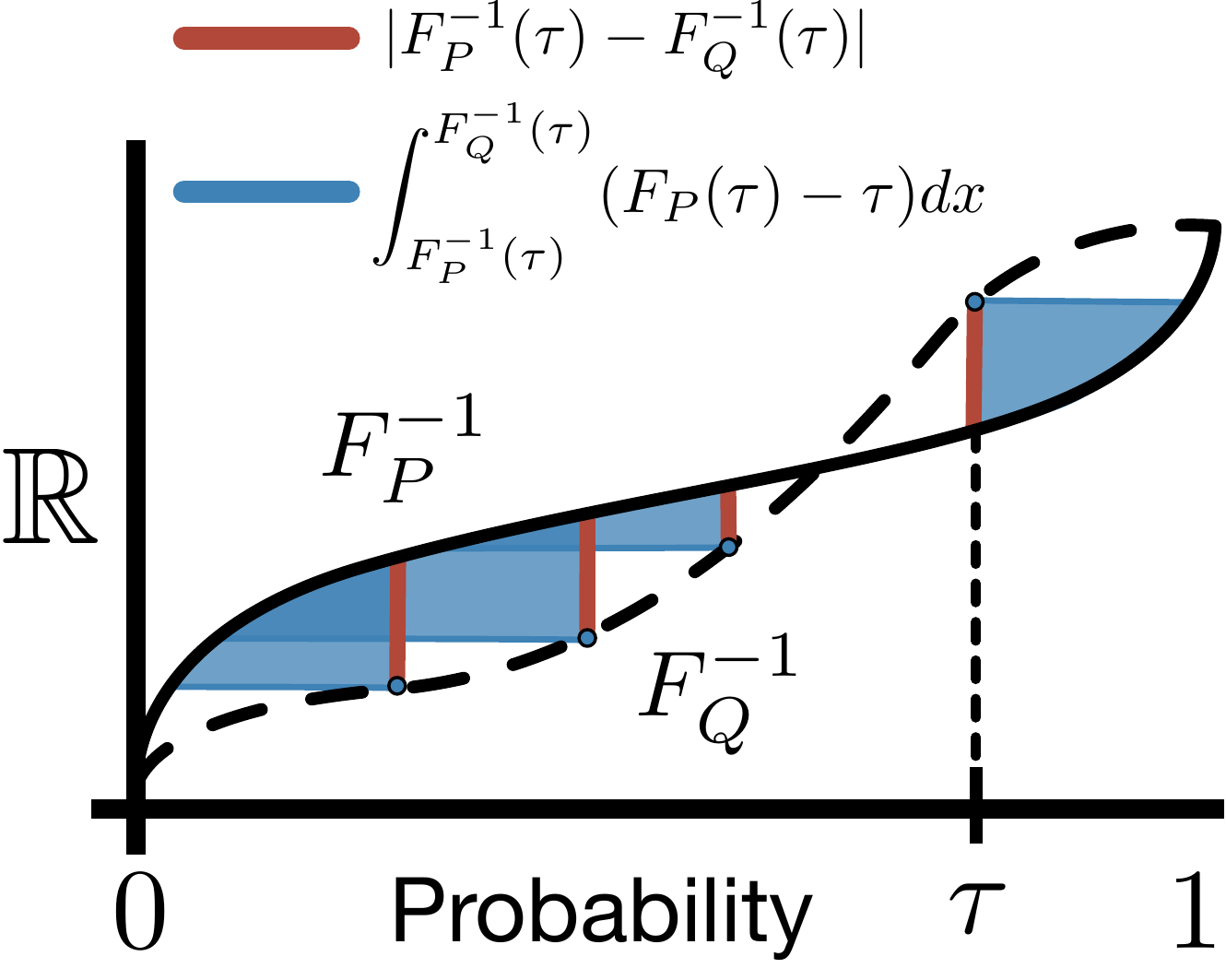}
\end{center}
\caption{Illustration of the relation between the $1$-Wasserstein metric (red) and the quantile divergence (blue).}\label{fig:qrmetric}
\end{figure}

We illustrate the relation between the $1$-Wasserstein metric and the quantile divergence in Figure~\ref{fig:qrmetric}. Notice that, for each $\tau \in [0, 1]$, while the Wasserstein measures the error between the two quantile functions, the quantile divergence measures a subset of the area enclosed between their graphs.

\section*{Network and Training Details}
% \label{sec:params}

All PixelCNN and PixelIQN models in Section~\ref{sec:pixeliqn} are directly based on the small and large conditional Gated PixelCNN models developed in \cite{van2016conditional}. For CIFAR-10 (Section~\ref{sec:cifar}), we are using the smaller variant with $15$ layer blocks, convolutional filters of size $5$, $128$ feature planes in each layer block, and $1024$ features planes for the residual connections feeding into the output layer of the network. For small ImageNet (Section~\ref{sec:imagenet}), we use both this model, and a larger $20$ layer version with $256$ feature planes in each layer block. 

For PixelIQN, we rescale the $\tau \in [0,1]^{3n^2}$ linearly to lie in $[-1, 1]^{3n^2}$, and input it to the network in exactly the same way as the location-dependent conditioning in \cite{van2016conditional}, that is, by applying a $1 \times 1$ convolution producing the same number of feature planes as in the respective layer block, and adding it to the output of this block prior to the gating activation.

All models on CIFAR-10 were trained for a total of $300$K steps, those on ImageNet for $400$K steps. We trained the small models with a mini-batch size of $32$, running approximately $200$K updates per day on a single NVIDIA Tesla P100 GPU, while the larger models were trained with a mini-batch size of $128$ with synchronous updates from $16$ P100 GPUs, achieving approximately half of this step rate.

\section*{Hyperparameter Tuning and Evaluation}
% \label{sec:tuning}

All quantitative evaluations of our PixelCNN and PixelIQN models are based on the Fr\'echet Inception Distance (FID) \cite{heusel2017gans}, 
\begin{equation*}
    d(x_1, x_2) = \| \mu_1 - \mu_2\|^2 + Tr(\Sigma_1 + \Sigma_2 - 2(\Sigma_1 \Sigma_2)^{1/2}),
\end{equation*}
where $(\mu_1, \Sigma_1)$ are the mean and covariance of $10,000$ samples from the model (PixelCNN or PixelIQN), and $(\mu_2, \Sigma_2)$ are the mean and covariance matrix computed over a set of $10,000$ training data points. We slightly deviate from the usual practice of using the entire training set for FID computation, as this would require an equal number ($50,000$ in the case of CIFAR-10) of samples to be drawn from the model, which is computationally very expensive for autoregressive models like PixelCNN or PixelIQN. 

We use Polyak averaging \cite{polyak1992acceleration}, keeping an exponentially weighted average over past parameters with a weight of $0.9999$. This average is being loaded instead of the model parameters before samples are generated, but never used for training.

To tune our small PixelCNN and PixelIQN models, we performed a hyperparameter search over $500$ hyperparameter configurations for each model, each configuration evaluated after $100$K training steps on CIFAR-10, based on its FID score computed on a small set of $2500$ generated samples.

For PixelCNN, the parameter search involved choosing from RMSProp, Adam, and SGD as the optimizer, and tuning the learning rate, involving both constant and decaying learning rate schedules. As a result we settled on the RMSProp optimizer and a set of three possible learning rate regimes, namely a constant learning rate of $10^{-4}$ or $3\cdot10^{-5}$, and a decaying learning rate regime: $10^{-4}$ in the first $120$K, $3\cdot10^{-5}$ for the next $60$K, and $10^{-5}$ for the remaining training steps. We found the first of these to work best on ImageNet, and the decaying schedule to work best on CIFAR-10, and only report the best model for each dataset.

For PixelIQN, the parameter search included the above (but with constant learning rates only), and additionally a sweep over a range of values for the Huber loss parameter $\kappa$ (Equation \ref{eqn:huberquantile}). As a result, we used Adam with a constant learning rate of $10^{-4}$ for all PixelIQN model variants on both datasets, and set $\kappa = 0.002$. We found that the model is not sensitive to this hyperparameter, but performs somewhat worse if the regular quantile regression loss is used instead of the Huber variant.

\section*{AIQN-VAE}
% \label{sec:iqnvae}

\begin{figure*}[t]
\begin{center}
\includegraphics[width=\textwidth]{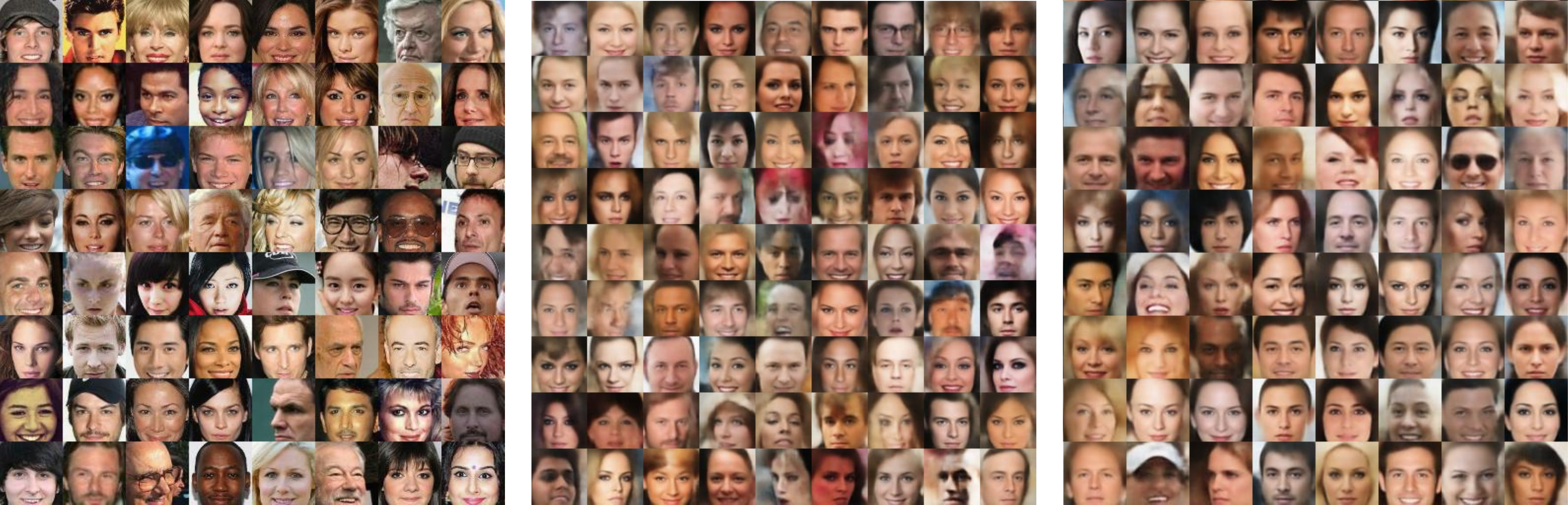}
\end{center}
\caption{CelebA 64x64: Real example images (left), samples generated by VAE (center), and samples generated by AIQN-VAE (right).}
\label{fig:iqnvae_celeba}
\end{figure*}

One potential drawback to PixelIQN presented above, shared by PixelCNN and more generally autoregressive models, is that due to their autoregressive nature sampling can be extremely time-consuming. This is especially true as the resolution of images increases. Although it is possible to partially reduce this overhead with clever engineering, these models are inherently much slower to sample from than models such as GANs and VAEs. In this section, we demonstrate how PixelIQN, due to the continuous nature of the quantile function, can be used to learn distributions over lower-dimensional, latent spaces, such as those produced by an autoencoder, variational or otherwise. Specifically, we use a standard VAE, but simultaneously train a small AIQN to model the training distribution over latent codes. For sampling, we then generate samples of the latent distribution using AIQN instead of the VAE prior.

This approach works well for two reasons. First, even a thoroughly trained VAE does not produce an encoder that fully matches the Gaussian prior. Generaly, the data distribution exists on a non-Gaussian manifold in the latent space, despite the use of variational training. Second, unlike existing methods, AIQN learns to approximate the full continuous-valued distribution without discretizing values or making prior assumptions about the value range or underlying distribution.

%distances in value have highly divergent implications on the final output of the decoder.

We can see similarities between this approach and two other recent publications. First, the $\alpha$-GAN proposed by \citet{rosca2017variational}. In both, there is an attempt to sample from the true latent distribution of a VAE-like latent variable model. % something wrong here, probably missing a second model?
In the case of $\alpha$-GAN this sampling distribution is trained using a GAN, while we propose to learn the distribution using quantile regression. The similarity makes sense considering AIQN shares some of the benefits of GANs. Unlike in this related work, we have not replaced the KL penalty on the latent representation. It would be an interesting direction for future research to explore a similar formulation. Generally, the same trade-offs between GANs and AIQN should be expected to come into play here just as they do when learning image distributions. Second, the VQ-VAE model \cite{van2017neural}, learns a PixelCNN model of the (discretized) latent space. Here, especially in the latent space, distribution losses respecting distances between individual points is more applicable than likelihood-based losses.

Let $e\colon \bR^{n} \to \bR^m$ and $d\colon \bR^m \to \bR^{n}$ be the mean of the encoder and decoder respectively of a VAE, although other forms of autoencoder could be substituted. Then, let $Q_\tau$ be an AIQN on the space $\bR^m$. During training we propose to minimize
\begin{equation*}
    \cL(x) = \cL_{VAE}(x) + \expect_{\tau \sim \U([0, 1]^m)} \rho_\tau^{\kappa}(e(x) - Q_\tau),
\end{equation*}
where $\cL_{VAE}$ is the standard VAE loss function. Then, for generation, we sample $\tau \sim \U([0, 1]^m)$, and reparameterize this sample through the AIQN and the decoder to produce $y = d(Q_\tau)$, a sample from the approximated distribution. We call this simple combination the AIQN-VAE.

\subsection*{CelebA}
% Changes from clean pull:
% z_dim = 2048 -> 32
% kl_weight = 2 -> 1
% gpu_id = 2 -> 1
% thirdparty slim vs normal import of slim

We demonstrate the AIQN-VAE using the CelebA dataset \cite{liu2015faceattributes}, at resolution $64 \times 64$. We modified an open source VAE implementation\footnote{https://github.com/LynnHo/VAE-Tensorflow} to simultaneously train the AIQN on the output of the VAE encoder, with Polyak averaging \cite{polyak1992acceleration}
of the AIQN weights. We reduce the latent dimension to $32$, as our purpose is to investigate the use of VAEs to learn in lower-dimensional latent spaces. The AIQN used three fully connected layers of width $512$ with ReLU activations. For the AIQN-VAE, but not the VAE, we lowered latent dimension variance to $0.1$ and the KL-term weight to $0.5$. It has been observed that in this setting the VAE prior alone will produce poor samples, thus high-quality samples will only be possible by learning the latent distribution. Figure~\ref{fig:iqnvae_celeba} shows samples from both a VAE and AIQN-VAE after $200K$ training iterations.
Both models may be expected to improve with further training, however, we can see that the AIQN-VAE samples are frequently clearer and less blurry than those from the VAE.

\begin{figure*}
\begin{center}
\includegraphics[width=\textwidth]{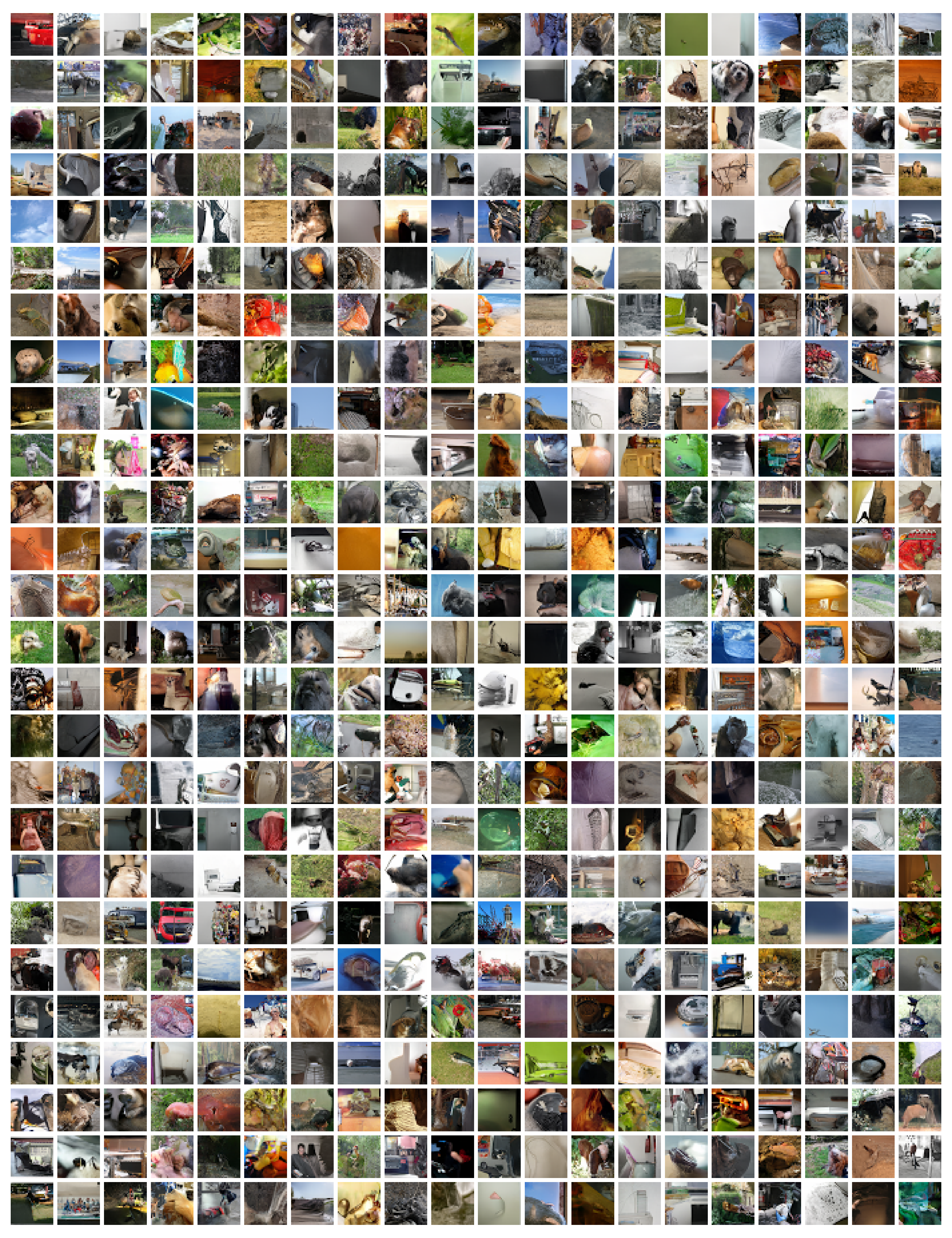} %iqn_samples_24x16.pdf}
\end{center}
\caption{Samples from PixelIQN trained on small ImageNet.}\label{fig:imagenet_more_samples}
\end{figure*}

\begin{figure*}
\begin{center}
\includegraphics[width=0.93\textwidth]{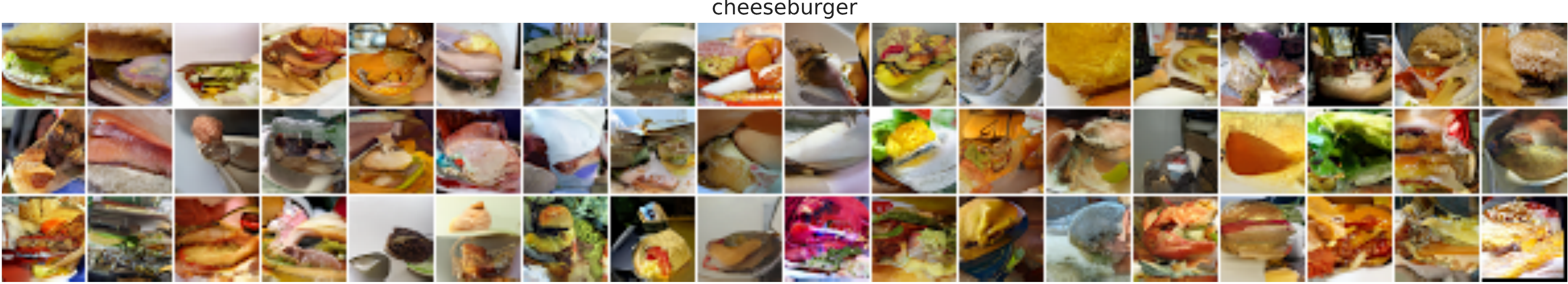}
\includegraphics[width=0.93\textwidth]{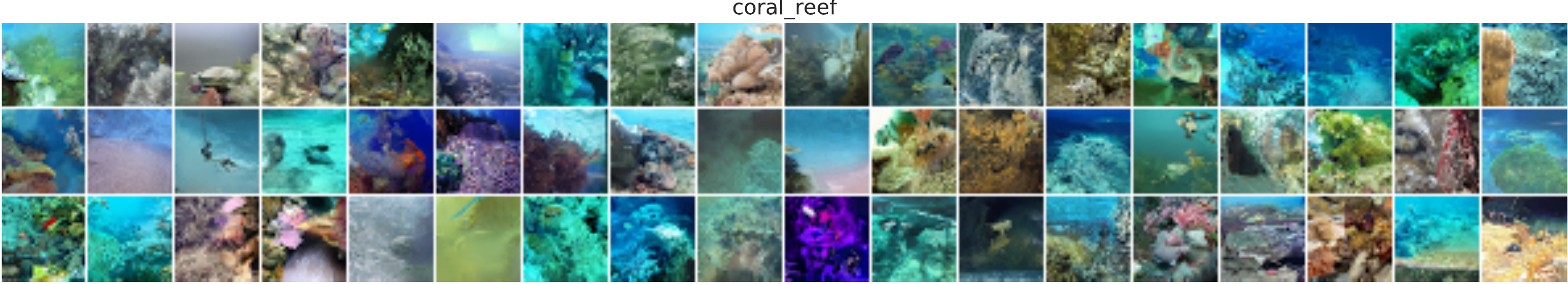}
\includegraphics[width=0.93\textwidth]{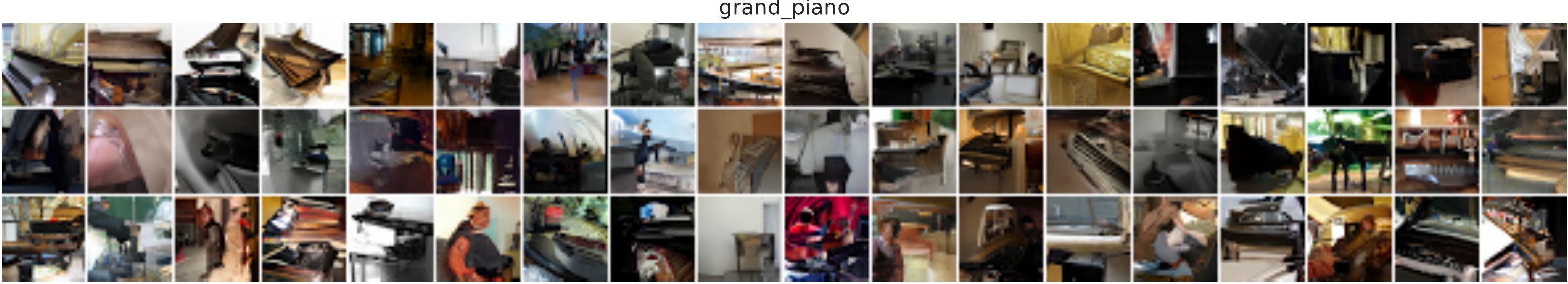}
\includegraphics[width=0.93\textwidth]{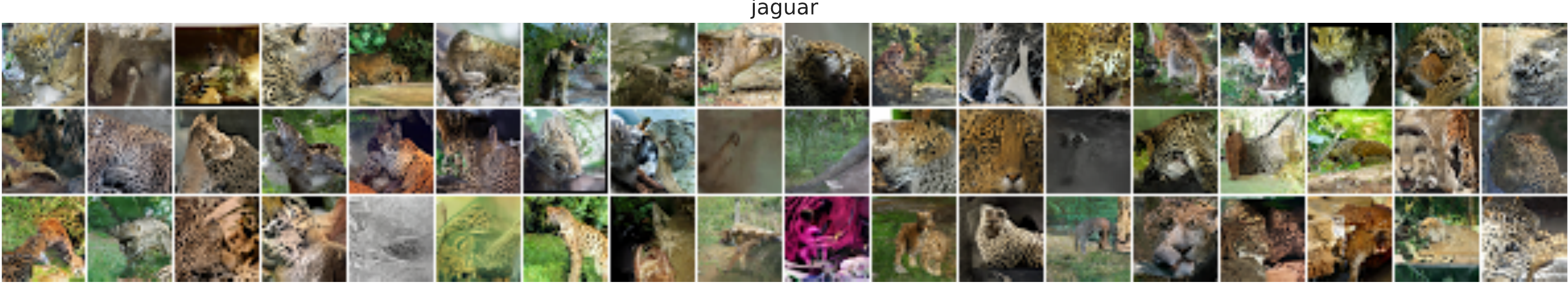}
\includegraphics[width=0.93\textwidth]{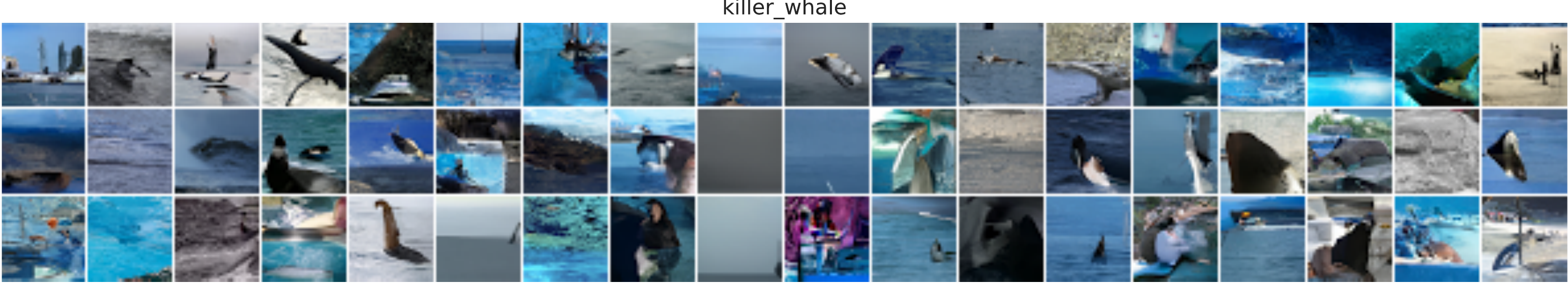}
\includegraphics[width=0.93\textwidth]{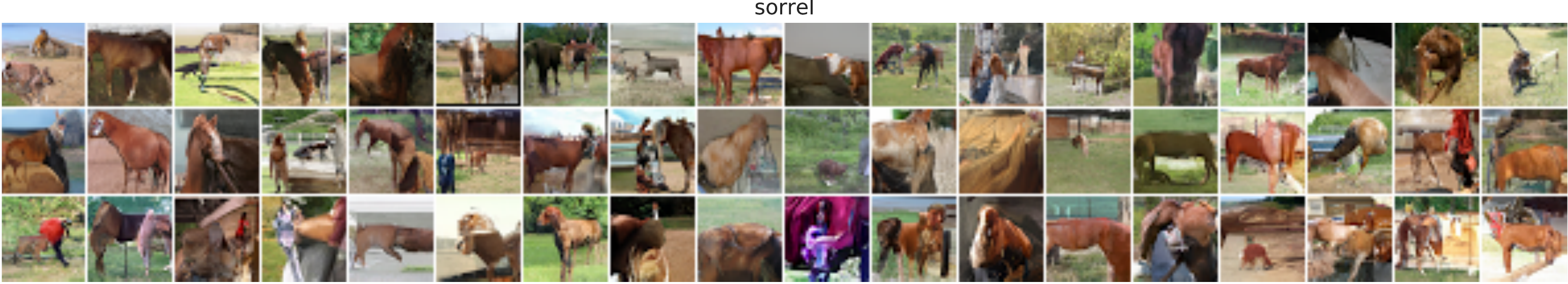}
\includegraphics[width=0.93\textwidth]{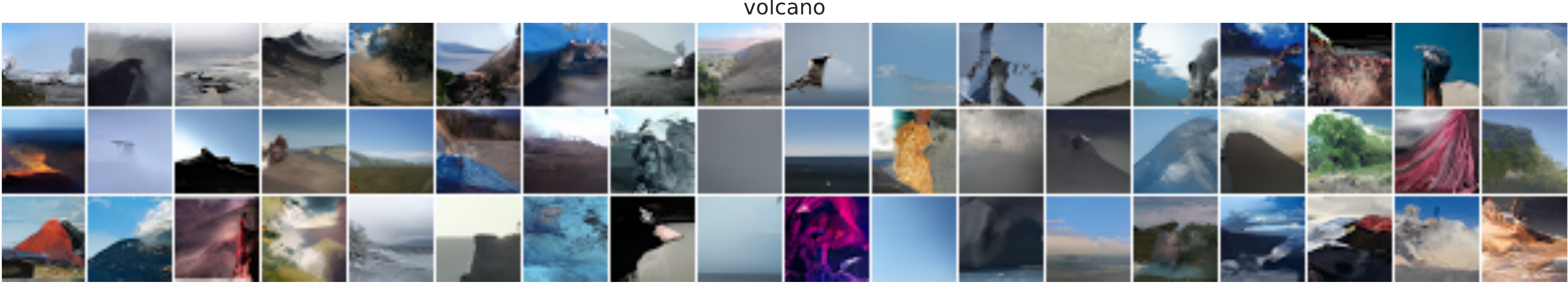}
\includegraphics[width=0.93\textwidth]{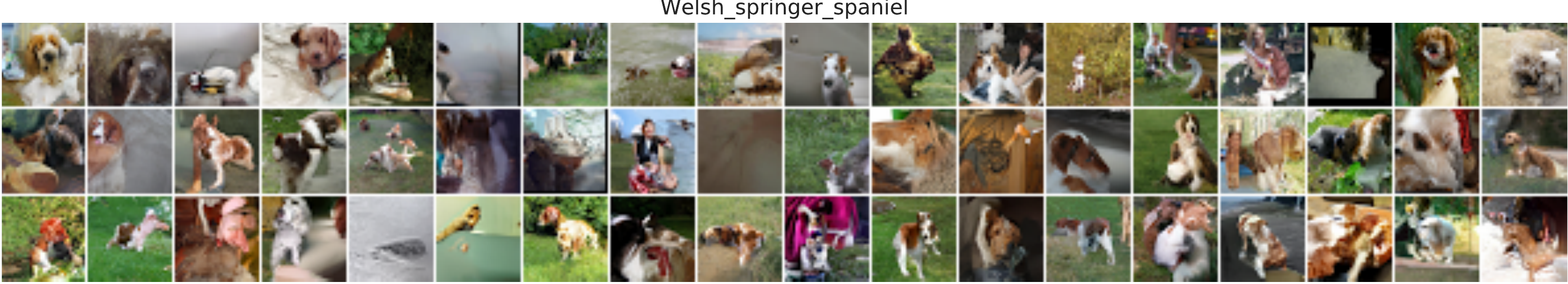}
\end{center}
\caption{Class-conditional samples from PixelIQN trained on small ImageNet.}\label{fig:cc_more_samples}
\end{figure*}

\begin{figure*}
\begin{center}
\includegraphics[width=.85\textwidth]{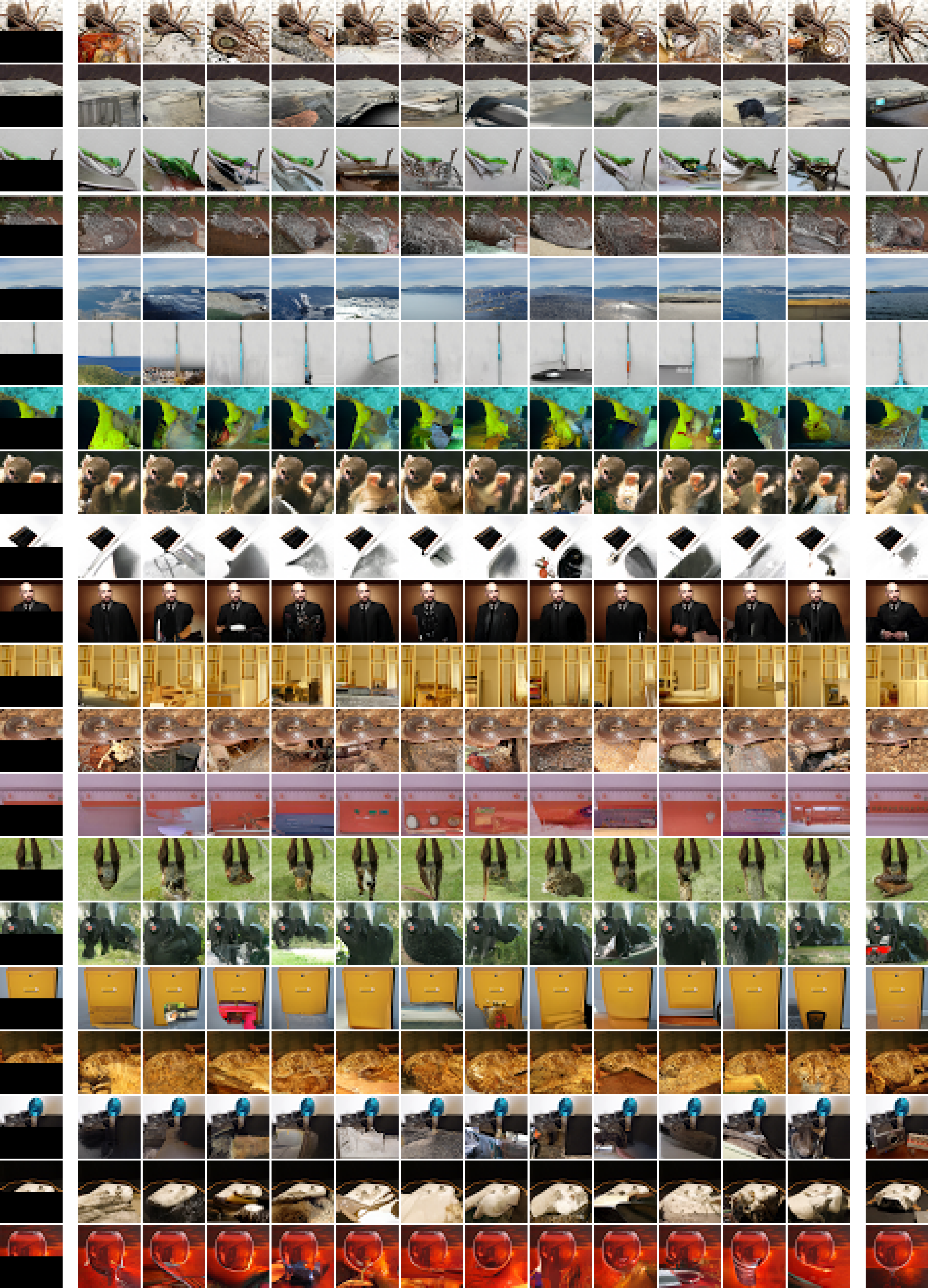}
\end{center}
\caption{Inpainting. Left column: Masked image given to the network. Middle columns: alternative image completions by the PixelIQN network for different values of $\tau$. Right column: Original image.}\label{fig:imagenet_more_inpainting}
\end{figure*}


\end{document}